\documentclass{article}

\usepackage{graphicx} 

\usepackage{amsfonts}       
\usepackage{nicefrac}       
\usepackage{microtype}      

\usepackage{algorithm}
\usepackage{algorithmic}

\usepackage{enumitem}
\usepackage{hyperref}

\usepackage{fullpage}
\usepackage{authblk}
\usepackage[round]{natbib} 

\usepackage{Definitions}
\usepackage{color}

\newcommand{\alg}{Smoothed Bellman Error Embedding}
\newcommand{\algabb}{SBEED}

\renewcommand{\leq}{\leqslant}
\renewcommand{\le}{\leqslant}
\renewcommand{\geq}{\geqslant}
\renewcommand{\ge}{\geqslant}
\renewcommand{\Vtil}{V_\lambda}

\newcommand{\Fsa}{\Fcal_{\Scal\times\Acal}}
\newcommand{\Pacal}{\Pcal_{\Acal}}
\newcommand{\KL}{\operatorname{KL}}

\title{\huge \algabb: Convergent Reinforcement Learning with Nonlinear Function Approximation}

\author{
  Bo Dai$^{1}$, Albert Shaw$^1$, Lihong Li$^2$, Lin Xiao$^3$, Niao He$^4$, Zhen Liu$^1$, Jianshu Chen$^5$, Le Song$^1$\\
  $^1$Georgia Insititute of Technology\\
  $^2$Google Inc., $^3$Microsoft Research, Redmond\\
  $^4$University of Illinois at Urbana Champaign, $^5$Tencent AI Lab, Bellevue\\
}

\begin{document}
\maketitle

\begin{abstract}
	When function approximation is used, solving the Bellman optimality equation with stability guarantees has remained a major open problem in reinforcement learning for decades.  The fundamental difficulty is that the Bellman operator may become an expansion in general, resulting in oscillating and even divergent behavior of popular algorithms like Q-learning.  In this paper, we revisit the Bellman equation, and reformulate it into a novel primal-dual optimization problem using Nesterov's smoothing technique and the Legendre-Fenchel transformation.  We then develop a new algorithm, called \emph{\alg}, to solve this optimization problem where any differentiable function class may be used.  We provide what we believe to be the first convergence guarantee for general nonlinear function approximation, and analyze the algorithm's sample complexity.  Empirically, our algorithm compares favorably to state-of-the-art baselines in several benchmark control problems.
\end{abstract}

\section{Introduction}\label{sec:intro}


In reinforcement learning (RL), the goal of an agent is to learn a policy that maximizes long-term returns by sequentially interacting with an unknown environment~\citep{SutBar98}.  The dominating framework to model such an interaction is the Markov decision process, or MDP, in which the optimal value function are characterized as a fixed point of the Bellman operator.  A fundamental result for MDP is that the Bellman operator is a contraction in the value-function space, so the optimal value function is the unique fixed point.  Furthermore, starting from any initial value function, iterative applications of the Bellman operator ensure convergence to the fixed point.  Interested readers are referred to the textbook of~\citet{Puterman14} for details.

Many of the most effective RL algorithms have their root in such a fixed-point view.  The most prominent family of algorithms is perhaps the temporal-difference algorithms, including TD$(\lambda)$~\citep{Sutton88Learning}, Q-learning~\citep{Watkins89Learning}, SARSA~\citep{Rummery94Online,Sutton96Generalization}, and numerous variants such as the empirically very successful DQN~\citep{Mnih15Human} and A3C~\citep{Mnih16Asynchronous} implementations.  Compared to direct policy search/gradient algorithms like REINFORCE~\citep{Williams92Simple}, these fixed-point methods make learning more efficient by \emph{bootstrapping} (a sample-based version of Bellman operator).

When the Bellman operator can be computed exactly (even on average), such as when the MDP has finite state/actions, convergence is guaranteed thanks to the contraction property~\citep{Bertsekas96Neuro}.  Unfortunately, when function approximatiors are used, such fixed-point methods \emph{easily} become unstable or even divergent~\citep{Boyan95Generalization,Baird95,Tsitsiklis97Analysis}, except in a few special cases. For example,
\begin{itemize}
\item for some rather restrictive function classes, such as those with a non-expansion property, some of the finite-state MDP theory continues to apply with modifications~\citep{Gordon95Stable,Ormoneit02Kernel,AntSzeMun08};
\item when \emph{linear} value function approximation in certain cases, convergence is guaranteed: for evaluating a \emph{fixed} policy from \emph{on-policy} samples~\citep{Tsitsiklis97Analysis}, for evaluating the policy using a closed-form solution from \emph{off-policy} samples~\citep{Boyan02Least,Lagoudakis03Least}, or for optimizing a policy using samples collected by a stationary policy~\citep{Maei10Toward}.
\end{itemize}
In recent years, a few authors have made important progress toward finding scalable, convergent TD algorithms, by designing proper objective functions and using stochastic gradient descent~(SGD) to optimize them~\citep{Sutton99Fast,Maei11Gradient}. Later on, it was realized that several of these gradient-based algorithms can be interpreted as solving a primal-dual problem~\citep{Mahadevan14Proximal,LiuLiuGhaMah15,Macua15Distributed,DaiHePanBooetal16}. This insight has led to novel, faster, and more robust algorithms by adopting sophisticated optimization techniques~\citep{Du17Stochastic}.  Unfortunately, to the best of our knowledge, all existing works either assume linear function approximation or are designed for policy evaluation.  It remains a major open problem how to find the \emph{optimal policy} reliably with general \emph{nonlinear} function approximators such as neural networks, especially in the presence of \emph{off-policy} data.

\paragraph{Contributions} In this work, we take a substantial step towards solving this decades-long open problem, leveraging a powerful saddle-point optimization perspective, to derive a new algorithm called \emph{\alg\ (\algabb) algorithm}. 
Our development hinges upon a novel view of a smoothed Bellman optimality equation, which is then transformed to the final primal-dual optimization problem. \algabb\ learns the optimal value function and a stochstic policy in the primal, and the Bellman error (also known as Bellman residual) in the dual.  By doing so, it avoids the non-smooth $\max$-operator in the Bellman operator, as well as the double-sample challenge that has plagued RL algorithm designs~\citep{Baird95}.  More specifically,
\begin{itemize}
\item {\algabb\ is stable for a broad class of nonlinear function approximators including neural networks, and provably converges to a solution with vanishing gradient. This holds even in the more challenging off-policy case;}
\item {it uses bootstrapping to yield high sample efficiency, as in TD-style methods, and is also generalized to cases of multi-step bootstrapping and eligibility traces;}
\item {it avoids the double-sample issue and directly optimizes the squared Bellman error based on sample trajectories;}
\item {it uses stochastic gradient descent to optimize the objective, thus very efficient and scalable.}
\end{itemize}
Furthermore, the algorithm handles both the optimal value function estimation and policy optimization in a unified way, and readily applies to both continuous and discrete action spaces.  We compare the algorithm with state-of-the-art baselines on several continuous control benchmarks, and obtain excellent results.

\section{Preliminaries}\label{sec:preliminary}

In this section, we introduce notation and technical background  that is needed in the rest of the paper. We denote a Markov decision process (MDP) as $\Mcal = \rbr{\Scal, \Acal, P, R, \gamma}$, where $\Scal$ is a (possible infinite) state space, $\Acal$ an action space, $P(\cdot|s, a)$ the transition probability kernel defining the distribution over next states upon taking action $a$ on state $s$, $R(s, a)$ the average immediate reward by taking action $a$ in state $s$, and  $\gamma\in (0, 1)$ a discount factor. Given an MDP, we wish to find a possibly stochastic policy $\pi: \Scal\to \Pacal$ to maximize the expected discounted cumulative reward starting from any state $s\in\Scal$: $\EE\sbr{\sum_{t=0}^\infty \gamma^t R(s_t, a_t) \Big| s_0=s,\pi}$, where $\Pacal$ denotes all probability measures over $\Acal$.  The set of all policies is denoted by $\Pcal\defeq (\Pacal)^\Scal$.

Define
$
V^*(s) \defeq \max_{\pi(\cdot|s)}\EE\sbr{\sum_{t = 0}^\infty \gamma^t R(s_t, a_t)|s_0=s,\pi}
$
to be the optimal value function.  It is known that $V^*$ is the unique fixed point of the Bellman operator $\Tcal$, or equivalently, the unique solution to the Bellman optimality equation (Bellman equation, for short)~\citep{Puterman14}:
\begin{eqnarray}\label{eq:bellman_opt}
V(s) = (\Tcal V)(s) \defeq \max_{a} \,{R(s, a) + \gamma \EE_{s'|s, a}\sbr{V(s')}}.
\end{eqnarray}
The optimal policy $\pi^*$ is related to $V^*$ by the following:
\begin{equation*} \label{eq:opt_policy}
\pi^*(a|s)  = \argmax_a \cbr{R(s, a) + \gamma \EE_{s'|s, a}\sbr{V^*(s')}}. 
\end{equation*}
It should be noted that in practice, for convenience we often work on the Q-function instead of the state-value function $V^*$.  In this paper, it suffices to use the simpler $V^*$ function.

\section{A Primal-Dual View of Bellman Equation}\label{sec:stoc_primal_dual}

In this section, we introduce a novel view of Bellman equation that enables the development of the new algorithm in \secref{sec:algorithm}.  After reviewing the Bellman equation and the challenges to solve it, we describe the two key technical ingredients that lead to our primal-dual reformulation.

We start with another version of Bellman equation that is equivalent to \eqnref{eq:bellman_opt} (see, e.g., \citet{Puterman14}):
\begin{equation}\label{eq:bellman_opt_dual}
V(s) = \max_{\pi(\cdot|s)\in\Pacal} \EE_{a\sim\pi(\cdot|s)}\sbr{{R(s, a)} + \gamma \EE_{s'|s, a}\sbr{V(s')}}. 
\end{equation}
\eqnref{eq:bellman_opt_dual} makes the role of a policy explicit.  Naturally, one may try to jointly optimize over $V$ and $\pi$ to minimize the discrepancy between the two sides of \eqref{eq:bellman_opt_dual}.  For concreteness, we focus on the square distance in this paper, but our results can be extended to other convex loss functions.  Let $\mu$ be some given state distribution so that $\mu(s)>0$ for all $s\in\Scal$.  Minimizing the \emph{squared Bellman error} gives the following:
\begin{eqnarray}
\min_V\, \EE_{s \sim \mu}\sbr{\rbr{\max_{\pi(\cdot|s)\in \Pacal} \EE_{a\sim\pi(\cdot|s)}\sbr{R(s,a) 
+ \gamma \EE_{s'|s, a}\sbr{V(s')}} - V(s)}^2}\,. \label{eq:mean_square_bellman_opt}
\end{eqnarray}

While natural, this approach has several major difficulties when it comes to optimization, which are to be dealt with in the following subsections:
\begin{itemize}
\item {The $\max$ operator over $\Pacal$ introduces non-smoothness to the objective function.  A slight change in $V$ may cause large differences in the RHS of \eqnref{eq:bellman_opt_dual}.}
\item {The conditional expectation, $\EE_{s'|s, a}\sbr{\cdot}$, composed within the square loss, requires double samples~\citep{Baird95} to obtain unbiased gradients, which is often impractical in most but simulated environments.}
\end{itemize}

\subsection{Smoothed Bellman Equation}\label{subsec:smoothing}

To avoid the instability and discontinuity caused by the $\max$ operator, we use the smoothing technique of \citet{Nesterov05} to smooth the Bellman operator $\Tcal$. Since policies are conditional distributions over $\Acal$, we choose entropy regularization, and \eqnref{eq:bellman_opt_dual} becomes:
\begin{eqnarray}
\Vtil(s) =\max_{\pi(\cdot|s)\in \Pacal} \Big(\EE_{a\sim\pi(\cdot|s)}\big(R(s, a) + \gamma \EE_{s'|s, a}\sbr{\Vtil(s')} \big) + \lambda H(\pi,s)\Big) \,, \label{eq:smoothed_bellman_opt_dual}
\end{eqnarray}
where $H(\pi,s) \defeq -\sum_{a\in\Acal}\pi(a|s)\log \pi(a|s)$, and $\lambda \ge 0$ controls the degree of smoothing.  Note that with $\lambda=0$, we obtain the standard Bellman equation.  Moreover, the regularization may be viewed a shaping reward added to the reward function of an induced, equivalent MDP; see \appref{appendix:multi_step} for more details. 

Since negative entropy is the conjugate of the log-sum-exp function~\citep[Example~3.25]{BoyVan04}, \eqnref{eq:smoothed_bellman_opt_dual} can be written equivalently as
\begin{eqnarray}\label{eq:smoothed_bellman_opt_dual2} 
\Vtil(s) = \rbr{\Tcal_\lambda \Vtil}(s) \defeq \lambda \log\rbr{\sum_{a\in\Acal}\exp\rbr{\frac{R(s, a) + \gamma \EE_{s'|s, a}\sbr{\Vtil(s')}}{\lambda}}}
\end{eqnarray}
where the $\log$-sum-$\exp$ is an effective smoothing approximation of the $\max$-operator.

\paragraph{Remark.} While \eqnsref{eq:smoothed_bellman_opt_dual} and \eq{eq:smoothed_bellman_opt_dual2} are inspired by Nestorov smoothing technique, they can also be derived from other principles~\citep{RawTouVij12,FoxPakTis15,NeuJonGom17,NacNorXuSch17a,AsaLit16}.  For example, \citet{NacNorXuSch17a} use entropy regularization in the policy space to encourage exploration, but arrive at the same smoothed form; the smoothed operator $\Tcal_\lambda$ is called ``Mellowmax'' by \citet{AsaLit16}, which is obtained as a particular instantiation of the quasi-arithmetic mean.  In the rest of the subsection, we review the properties of $\Tcal_\lambda$, although some of the results have appeared in the literature in slightly different forms.  Proofs are deferred to \appref{appendix:alg_derivation_proof}.

First, we show $\Tcal_\lambda$ is also a contraction, as with the standard Bellman operator~\citep{FoxPakTis15,AsaLit16}:
\begin{proposition}[Contraction]\label{thm:contraction}
$\Tcal_\lambda$ is a $\gamma$-contraction.  Consequently, the corresponding smoothed Bellman equation~\eq{eq:smoothed_bellman_opt_dual}, or equivalently \eq{eq:smoothed_bellman_opt_dual2}, has a unique solution $\Vtil^*$.
\end{proposition}

Second, we show that while in general $V^* \neq \Vtil^*$, their difference is controlled by $\lambda$.  To do so, define $H^* \defeq \max_{s\in\Scal,\pi(\cdot|s)\in\Pacal} H(\pi,s)$.  For finite action spaces, we immediately have $H^* = \log(\abr{\Acal})$. 
\begin{proposition}[Smoothing bias]\label{lemma:smooth_bias}  Let $V^*$ and $\Vtil^*$ be fixed points of~\eq{eq:bellman_opt_dual} and~\eq{eq:smoothed_bellman_opt_dual}, respectively. Then,
\begin{equation*}
\nbr{V^*(s) - \Vtil^*(s)}_\infty \le \frac{\lambda H^*}{1 - \gamma} \,.
\end{equation*}
Consequently, as $\lambda\to 0$, $\Vtil^*$ converges to $V^*$ pointwisely.
\end{proposition}

Finally, the smoothed Bellman operator has the very nice property of temporal consistency~\citep{RawTouVij12,NacNorXuSch17a}:
\begin{proposition}[Temporal consistency]\label{thm:smoothed_bellman_opt_pi}
Assume $\lambda>0$.  Let $\Vtil^*$ be the fixed point of~\eq{eq:smoothed_bellman_opt_dual} and  $\pi_\lambda^*$ the corresponding policy that attains the maximum on the RHS of~\eq{eq:smoothed_bellman_opt_dual}. Then, $(\Vtil^*,\pi_\lambda^*)$ is the unique $(V,\pi)$ pair that satisfies the following equality for all $(s,a)\in\Scal \times \Acal$:
\begin{equation}\label{eq:smoothed_bellman_opt_pi}
V(s)={R(s, a) + \gamma\EE_{s'|s, a}\sbr{V(s')} - \lambda\log\pi(a|s)} \,.
\end{equation}
\end{proposition}
In other words, \eqnref{eq:smoothed_bellman_opt_pi} provides an easy-to-check condition to characterize the optimal value function and optimal policy on \emph{arbitrary} pair of $(s, a)$, therefore, which is easy to incorporate \emph{off-policy} data. It can also be extended to the multi-step or eligibility-traces cases (\appref{appendix:algrithm_derivation}; see also \citet[Chapter~7]{SutBar98}).  Later, this condition will be one of the critical foundations to develop our new algorithm.

\subsection{Bellman Error Embedding}\label{subsec:saddle_point}

A natural objective function inspired by \eqref{eq:smoothed_bellman_opt_pi} is the \emph{mean squared consistency Bellman error}, given by:
\begin{eqnarray}\label{eq:simplified_smooth_bellman_loss}
\min_{V, \pi\in\Pcal} \; \ell(V, \pi) &\defeq& \EE_{s,a}\Big[\big(R(s, a) + \gamma \EE_{s'|s, a}\sbr{V(s')} - \lambda\log\pi(a|s) - V(s)\big)^2\Big]\,,
\end{eqnarray}
where $\EE_{s,a}[\cdot]$ is shorthand for $\EE_{s\sim\mu(\cdot),a\sim\pi_b(\cdot|s)}[\cdot]$.
Unfortunately, due to the inner conditional expectation, it would require two independent sample of $s'$ (starting from the same $(s,a)$) to obtain an unbiased estimate of gradient of $f$, a problem known as the double-sample issue~\citep{Baird95}.  In practice, however, one can rarely obtain two independent samples except in simulated environments.

To bypass this problem, we make use of the conjugate of the square function~\citep{BoyVan04}: ${{x}^2} = \max_\nu \rbr{2\nu x -\nu^2}$, as well as the interchangeability principle~\citep{shapiro09lectures,DaiHePanBooetal16} to rewrite the optimization problem~\eq{eq:simplified_smooth_bellman_loss} into an equivalent form:
\begin{eqnarray} \label{eq:dual_simplified_smooth_bellman}
\min_{V, \pi\in\Pcal}\max_{\nu\in\Fsa} L(V, \pi; \nu) \defeq 2\EE_{s, a, s'}\Big[ \nu(s, a) \big( R(s, a) + \gamma {V(s') - \lambda\log\pi(a|s)}- V(s) \big) \Big] - \EE_{s, a, s'}\sbr{\nu^2(s, a)}\,,\end{eqnarray}
where $\Fsa$ is the set of real-valued functions on $\Scal\times \Acal$, $\EE_{s,a,s'}[\cdot]$ is shorthand for $\EE_{s\sim\mu(\cdot),a\sim\pi_b(\cdot|s),s'\sim P(\cdot|s,a)}[\cdot]$.  Note that \eq{eq:dual_simplified_smooth_bellman} is not a standard convex-concave saddle-point problem: the objective is convex in $V$ for any fixed $(\pi,\nu)$, and concave in $\nu$ for any fixed $(V,\pi)$, but not necessarily convex in $\pi\in \Pcal$ for any fixed $(V,\nu)$. 

\paragraph{Remark.} In contrast to our saddle-point formulation~\eq{eq:dual_simplified_smooth_bellman}, \citet{NacNorXuSch17a} get around the double-sample obstacle by minimizing an upper bound of $\ell(V, \pi)$: $\tilde{\ell}(V,\pi):=\EE_{s, a, s'}\sbr{\rbr{R(s, a) + \gamma {V(s')} - \lambda\log\pi(a|s)-V(s)}^2}$.  As is known~\citep{Baird95}, the gradient of $\tilde{\ell}$ is different from that of $f$, as it has a conditional variance term coming from the stochastic outcome $s'$.  In problems where this variance is highly heterogeneous across different $(s,a)$ pairs, impact of such a bias can be substantial.

Finally, substituting the dual function $\nu(s, a) = \rho(s, a) - V(s)$, the objective in the saddle point problem becomes
\begin{eqnarray}\label{eq:variance_reduction}
\min_{V, \pi}\max_{\rho\in\Fsa}L_1(V,\pi;\rho) &\defeq& \EE_{s,a,s'}\sbr{\rbr{\delta(s,a, s')- V(s)}^2} -\EE_{s, a, s'}\sbr{\rbr{{\delta(s,a, s')}- \rho(s, a)}^2}
\end{eqnarray}
where $\delta(s,a, s') \defeq R(s, a) + \gamma {V(s') - \lambda\log\pi(a|s)}$. Note that the first term is $\tilde{\ell}(V,\pi)$, and the second term will cancel the extra variance term (see \propref{thm:variance_cancellation} in \appref{appendix:variance_reduciton}).  The use of an auxiliary function to cancel the variance is also observed by~\citet{AntSzeMun08}. On the other hand, when function approximation is used, extra bias will also be introduced. We note that such a saddle-point view of debiasing the extra variance term leads to a useful mechanism for better bias-variance trade-offs, leading to the final primal-dual formulation we aim to solve in the next section:
\begin{eqnarray}\label{eq:one_step_trade_off}
\min_{V, \pi\in\Pcal} \max_{\rho\in\Fsa}\hspace{-2mm}L_\eta(V,  \pi; \rho) :=\EE_{s,a,s'}\sbr{\rbr{\delta(s,a, s')- V(s)}^2}-\eta\EE_{s, a, s'}\sbr{\rbr{{\delta(s,a, s')}- \rho(s, a)}^2}\,,
\end{eqnarray}
where $\eta\in[0,1]$ is a hyper-parameter controlling the trade-off.  When $\eta=1$, this reduces to the original saddle-point formulation~\eq{eq:dual_simplified_smooth_bellman}.  When $\eta=0$, this reduces to the surrogate objective considered by~\citet{NacNorXuSch17a}.

\section{\alg}\label{sec:algorithm}

In this section, we derive the Smoothed Bellman Error EmbeDding~(\algabb) algorithm, based on stochastic mirror descent~\citep{NemJudLanSha09}, to solve the smoothed Bellman equation.  For simplicity of exposition, we mainly discuss the one-step optimization~\eq{eq:one_step_trade_off}, although it is possible to generalize the algorithm to the multi-step and eligibility-traces settings; see Appendices~\ref{appendix:multi_step} and~\ref{appendix:eligibility_trace} for details.

Due to the curse of dimensionality, the quantities $(V,\pi,\rho)$ are often represented by compact, parametric functions in practice. Denote these parameters by $w=(w_V, w_\pi, w_\rho)$. Abusing notation a little bit, we now write the objective function $L_\eta(V,\pi;\rho)$ as $L_\eta(w_V, w_\pi; w_\rho)$.  

\newcommand{\dsas}{\delta_{s,a,s'}}

First, we note that the inner (dual) problem is standard least-squares regression with parameter $w_\rho$, so can be solved using a variety of algorithms~\citep{Bertsekas16Nonlinear}; in the presence of special structures like convexity, global optima can be found efficiently~\citep{BoyVan04}.  The more involved part is to optimize the primal $(w_V,w_\pi)$, whose gradients are given by the following theorem.
\begin{theorem}[Primal gradient] \label{thm:gradient_estimator}
Define $\bar{\ell}_\eta(w_V, w_\pi) \defeq L_\eta(w_V, w_\pi; w_\rho^*)$, where $w_\rho^* = \arg\max_{w_\rho} L_\eta(w_V, w_\pi; w_\rho)$.  Let $\dsas$ be a shorthand for $\delta(s,a,s')$, and $\hat{\rho}$ be dual parameterized by $w_\rho^*$.  Then,
\begin{align*}
\nabla_{w_V} \bar{\ell}_\eta =& 2 \EE_{s,a,s'}\sbr{\rbr{\dsas - V(s)}\rbr{\gamma { \nabla_{w_V} V(s')} - \nabla_{w_V} V(s)}}  - 2 \eta\gamma\EE_{s,a,s'}\sbr{\rbr{\dsas - \hat{\rho}(s, a)}\nabla_{w_V} V(s')}\,, \\
\nabla_{w_\pi} \bar{\ell}_\eta =& -2\lambda\EE_{s, a, s'}\big[\rbr{(1 - \eta)\dsas + \eta\hat{\rho}(s, a) - V(s)} \cdot \nabla_{w_\pi}\log\pi(a|s)\big]\,. 
\end{align*}
\end{theorem}
With gradients given above, we may apply stochastic mirror descent to update $w_V$ and $w_\pi$; that is, given a stochastic gradient direction (for either $w_V$ or $w_\pi$), we solve the following prox-mapping in each iteration,
\begin{eqnarray*}
P_{z_V}(g) &=& \argmin_{w_V} \inner{w_V}{g} + D_V(w_V, z_V),\\
P_{z_\pi}(g) &=& \argmin_{w_\pi} \inner{w_\pi}{g} + D_\pi(w_\pi, {z_\pi}).
\end{eqnarray*}
where $z_V$ and $z_\pi$ can be viewed the current weight, and $D_V(w, z)$ and $D_\pi(w, z)$ are Bregman divergences.  We can use Euclidean metric for both $w_V$ and $w_\pi$, and possibly KL-divergence for $w_\pi$. The per-iteration computation complexity is therefore very low, and the algorithm can be scaled up to complex nonlinear approximations.

\algref{alg:sbeed} instantiates \algabb, combined with experience replay~\citep{Lin92Self} for greater data efficiency, in an online RL setting. New samples are added to the experience replay buffer $\Dcal$ at the beginning of each episode (Lines~3--5) with a behavior policy. Lines~6--11 correspond to the stochastic mirror descent updates on the primal parameters.  Line~12 sets the behavior policy to be the current policy estimate, although other choices may be used.  For example, $\pi_b$ can be a fixed policy~\citep{AntSzeMun08}, which is the case we will analyze in the next section.

\begin{algorithm}[t] 
\caption{{\small Online \algabb~learning with experience replay}} \label{alg:sbeed}
  \begin{algorithmic}[1]
    \STATE Initialize $w=(w_V,w_\pi,w_\rho)$ and $\pi_b$ randomly, set $\epsilon$.
    \FOR{episode $i=1,\ldots, T$}
      \FOR{size $k=1,\ldots, K$}
        \STATE Add new transition $(s, a, r, s')$ into $\Dcal$ by executing behavior policy $\pi_b$.
      \ENDFOR
      \FOR{iteration $j=1, \ldots, N$}
          \STATE Update $w_\rho^{j}$ by solving 
          \vspace{-3mm}
          $$\min_{w_\rho}~\widehat\EE_{\cbr{s, a, s'}\sim\Dcal}\sbr{\rbr{{\delta(s,a, s')}- \rho(s, a)}^2}.$$
          \vspace{-3mm}
          \STATE Decay the stepsize $\zeta_j$ in rate $\Ocal(1/j)$.
            \STATE Compute the stochastic gradients w.r.t. $w_V$ and $w_\pi$ as $\widehat \nabla_{w_V} \bar\ell(V, \pi)$ and $\widehat \nabla_{w_\pi} \bar\ell(V, \pi)$.
            \STATE Update the parameters of primal function by solving the prox-mappings, \ie,
            \vspace{-2mm}
            \begin{eqnarray*}
            &&\text{update $V$: }\quad w^j_V = P_{w^{j-1}_V}(\zeta_j \widehat \nabla_{w_V} \bar\ell(V, \pi))\\[-2mm]
            &&\text{update $\pi$: }\quad w^j_\pi = P_{w^{j-1}_\pi}(\zeta_j \widehat \nabla_{w_\pi} \bar\ell(V, \pi))
            \end{eqnarray*}
            \vspace{-6mm}
        \ENDFOR
        \STATE Update behavior policy $\pi_b = \pi^N$.
    \ENDFOR
  \end{algorithmic}
\end{algorithm}

\paragraph{Remark (Role of dual variables):}  The dual variable is obtained by solving
{\small
\begin{equation*}
\min_{\rho} \EE_{s, a, s'}\sbr{\rbr{{R(s, a) + \gamma V(s') - \lambda \log \pi(a|s)} - \rho(s, a)}^2}\,.
\end{equation*}
}
The solution to this optimization problem is
$$
\rho^*(s, a) = R(s, a) + \gamma \EE_{s'|s, a}\sbr{V(s')} - \lambda \log \pi(a|s)\,.
$$
Therefore, the dual variables try to approximate the one-step smoothed Bellman backup values, given a $(V,\pi)$ pair.  Similarly, in the equivalent \eq{eq:dual_simplified_smooth_bellman}, the optimal dual variable $\nu(s,a)$ is to fit the one-step smoothed Bellman error. Therefore, each iteration of \algabb\ could be understood as first fitting a parametric model to the one-step Bellman backups (or equivalently, the one-step Bellman error), and then applying stochastic mirror descent to adjust $V$ and $\pi$. 

\paragraph{Remark (Connection to TRPO and NPG):} The update of $w_\pi$ is related to trust region policy optimization~(TRPO)~\citep{SchLevAbbJoretal15} and natural policy gradient~(NPG)~\citep{Kakade02,RajLowTodKak17} when $D_\pi$ is the KL-divergence. Specifically, in~\citet{Kakade02} and \citet{RajLowTodKak17}, $w_\pi$ is update by $\argmin_{w_\pi} \EE\sbr{\inner{w_\pi}{\nabla_{w_\pi }\log \pi^t (a|s)A(a, s)}} + \frac{1}{\eta} \KL(\pi_{w_\pi}||\pi_{w^{old}_\pi})$, which is similar to $P_{w^{j-1}_\pi}$ with the difference in replacing the $\log \pi^t (a|s)A(a, s)$ with our gradient.  In~\citet{SchLevAbbJoretal15}, a related optimization with hard constraints is used for policy updates:
$\min_{w_\pi} \EE\sbr{\pi(a|s)A(a, s)}$, such that $\KL(\pi_{w_\pi}||\pi_{w^{old}_\pi})\le \eta$.
Although these operations are similar to $P_{w^{j-1}_\pi}$, we emphasize that the estimation of the advantage function, $A(s, a)$, and the update of policy are separated in NPG and TRPO. Arbitrary policy evaluation algorithm can be adopted for estimating the value function for \emph{current} policy. While in our algorithm, ${(1 - \eta)\delta(s, a) + \eta\rho^*(s, a) - V(s)}$ is different from the vanilla advantage function, which is designed appropriate for off-policy particularly, and the estimation of $\rho(s, a)$ and $V(s)$ is also integrated as the whole part.

\section{Theoretical Analysis}\label{sec:analysis}

In this section, we give a theoretical analysis for our algorithm in the same setting of~\citet{AntSzeMun08} where samples are prefixed and from \emph{one single $\beta$-mixing off-policy sample path}. For simplicity, we consider the case that applying the algorithm for $\eta = 1$ with the equivalent optimization~\eq{eq:dual_simplified_smooth_bellman}. The analysis is applicable to~\eq{eq:variance_reduction} directly. There are three groups of results. First, in \secref{sec:optimization error}, we show that under appropriate choices of stepsize and prox-mapping,  SBEED converges to a stationary point of the finite-sample approximation (i.e., empirical risk) of the optimization \eq{eq:dual_simplified_smooth_bellman}.  Second, in \secref{sec:statistical error}, we analyze generalization error of \algabb.  Finally,  in \secref{sec:main}, we give an overall performance bound for the algorithm, by combining four sources of errors: (i) optimization error, (ii) generalization error, (iii) bias induced by Nesterov smoothing, and (iv) approximation error induced by using function approximation. 

\paragraph{Notations.} 
Denote by $\Vcal_w,\, \Pcal_w$ and $\Hcal_w$ the parametric function classes of value function $V$, policy $\pi$, and dual variable $\nu$, respectively. Denote the total number of steps in the given off-policy trajectory as $T$. We summarize the notations for the objectives after parametrization and finite-sample approximation and their corresponding optimal solutions in the table for reference:  

\begin{table}[h]
\center
\begin{tabular}{|c||c|c|c|}
\hline
& minimax obj.& primal obj. &optimum\\
\hline
original &$L(V,\pi;\nu)$ & $\ell(V,\pi)$ & $(V_\lambda^*,\pi_\lambda^*)$\\
parametric&$L_w(V_w,\pi_w;\nu_w)$& $\ell_w(V_w,\pi_w)$ & $(V_w^*,\pi_w^*)$ \\
empirical&$\widehat L_T(V_w,\pi_w;\nu_w)$& $\widehat{\ell}_T(V_w,\pi_w)$ & $(\widehat V_w^*,\widehat \pi_w^*)$ \\
\hline
\end{tabular}
\end{table}
Denote the $L_2$ norm of a function $f$ w.r.t. $\mu(s)\pi_b(a|s)$ by $\nbr{f}^2 \defeq \int f(s, a)^2\mu(s)\pi_b(a|s)dsda$. We introduce a scaled norm : 
$
\nbr{V}^2_{\mu\pi_b} = \int \rbr{\gamma \EE_{s'|s, a}\sbr{V(s')} - V(s)}^2 \mu(s)\pi_b(a|s)dsda;
$
for value function; this is indeed a well-defined norm since $\nbr{V}^2_{\mu\pi_b} = \nbr{(\gamma P - I)V}_2^2$ and $I - \gamma P$ is injective. 

\subsection{Convergence Analysis}\label{sec:optimization error} 

It is well-known that for convex-concave saddle point problems, applying stochastic mirror descent ensures global convergence in a sublinear rate~\citep{NemJudLanSha09}. However, this no longer holds for problems without convex-concavity. Our SBEED algorithm, on the other hand, can be regarded as a special case of the stochastic mirror descent algorithm for solving the non-convex primal minimization problem $\min_{V_w, \pi_w} \widehat\ell_{T}(V_w,\pi_w)$. The latter was proven to converge sublinearly to a stationary point when stepsize is diminishing and Euclidean distance is used for the prox-mapping~\citep{GhaLan13}. For completeness, we list the result below.
\begin{theorem}[Convergence, \citet{GhaLan13}]\label{thm:convergence_opt}
Consider the case when Euclidean distance is used in the algorithm. Assume that the parametrized objective $\widehat\ell_{T}(V_w,\pi_w)$ is $K$-Lipschitz and variance of its stochastic gradient is bounded by $\sigma^2$. Let the algorithm run for $N$ iterations with stepsize $\zeta_k=\min\{\frac{1}{K}, \frac{D'}{\sigma\sqrt{N}}\}$ for some $D'>0$ and output $w^1,\ldots, w^N$. Setting the candidate solution to be $(\widehat V_w^N,\widehat{\pi}_w^N)$ with $w$ randomly chosen from $w^1,\ldots, w^N$ such that $P(w=w^j)=\frac{2\zeta_j-K\zeta_j^2}{\sum_{j=1}^N(2\zeta_j-K\zeta_j^2)}$, then it holds that
$\EE\sbr{\nbr{\nabla \widehat\ell_T(\widehat V_w^N,\widehat{\pi}_w^N)}^2}\leq \frac{K D^2}{N}+ (D'+\frac{D}{D'})\frac{\sigma}{\sqrt{N}}$
where $D:=\sqrt{2(\widehat\ell_T(V_w^1,\pi_w^1) -\min \widehat\ell_T(V_w,\pi_w))/K}$ represents the distance of the initial solution to the optimal solution. 
\end{theorem}
The above result implies that the algorithm converges sublinearly to a stationary point, whose rate will depend on the smoothing parameter. 

In practice, once we parametrize the dual function, $\nu$ or $\rho$, with neural networks, we cannot achieve the optimal parameters. However, we can still achieve convergence by applying the stochastic gradient descent to a (statistical) local Nash equilibrium asymptotically. We provided the detailed~\algref{alg:sbeed_dual_nn} and the convergence analysis in Appendix~\ref{appsec:convergence analysis}.

\subsection{Statistical Error}\label{sec:statistical error}
In this section, we characterize the statistical error, namely, $\epsilon_{\text{stat}}(T):=\ell_w(\widehat V_w^*, \widehat{\pi}_w^*) - \ell_w(V_w^*, \pi_w^*)$, induced by learning with finite samples. We first make the following standard assumptions about the MDPs:
\begin{assumption}[MDP regularity]\label{asmp:mdp_reg}
Assume $\nbr{R(s, a)}_\infty\le C_R$ and that there exists an optimal policy, $\pi_\lambda^*(a|s)$, such that $\nbr{\log\pi_\lambda^*(a|s)}_\infty\le C_\pi$. 
\end{assumption}
\begin{assumption}[Sample path property, ~\citet{AntSzeMun08}]\label{asmp:sample_measure}
Denote $\mu(s)$ as the stationary distribution of behavior policy $\pi_b$ over the MDP. We assume $\pi_b(a|s) > 0$, $\forall \rbr{s, a}\in \Scal\times\Acal$, and the corresponding Markov process $P^{\pi_b}(s'|s)$ is ergodic. We further assume that $\cbr{s_i}_{i=1}^T$ is strictly stationary and exponentially $\beta$-mixing with a rate defined by the parameters $\rbr{ b, \kappa}$\footnote{A $\beta$-mixing process is said to mix at an exponential rate with parameter $b,\kappa>0$ if $\beta_m=O(\exp(-bm^{-\kappa}))$.}.
\end{assumption}
Assumption~\ref{asmp:mdp_reg} ensures the solvability of the MDP and boundedness of the optimal value functions, $V^*$ and $\Vtil^*$. Assumption~\ref{asmp:sample_measure} ensures $\beta$-mixing property of the samples $\cbr{\rbr{s_i, a_i, R_i}}_{i=1}^T$ (see e.g., Proposition 4 in~\citet{CarChe02}) and is often necessary to prove large deviation bounds.

Invoking a generalized version of Pollard's tail inequality to $\beta$-mixing sequences and prior results in \citet{AntSzeMun08} and \citet{Haussler95}, we show that
\begin{theorem}[Statistical error]\label{thm:statistical_error} Under Assumption 2, it holds with at least probability $1 - \delta$,
$$
\epsilon_{\text{stat}}(T)\le 2\sqrt\frac{M\rbr{\max\rbr{M/b, 1}}^{1 / \kappa}}{C_2T}, 
$$
where $M, C_2$ are some constants. 
\end{theorem}
Detailed proof can be found in \appref{appsec:statistical error}.

\subsection{Error Decomposition}\label{sec:main}

As one shall see, the error between $(\widehat V_w^N,\widehat w^N)$ (optimal solution to the finite sample problem) and the true optimal $(V^*,\pi^*)$ to the Bellman equation consists three parts: i) the error introduced by smoothing, which has been characterized in Section~\ref{subsec:smoothing}, ii) the approximation error, which is tied to the flexibility of the parametrized function classes  $\Vcal_w,\, \Pcal_w$, $\Hcal_w$, and iii) the statistical error. More specifically, we arrive at the following explicit decomposition: 

Denote $\epsilon_{\text{app}}^\pi:= \sup_{\pi\in \Pcal}\inf_{\pi'\in\Pcal_w}\nbr{\pi - \pi'}_\infty$  as the function approximation error between $\Pcal_w$ and $\Pcal$. Denote $\epsilon_{\text{app}}^V$ and $\epsilon_{\text{app}}^\nu$ as approximation errors for $V$ and $\nu$, accordingly. More specifically, we arrive at 
\begin{theorem}\label{thm:error_decomposition} 
{Under Assumptions 1 and 2, it holds that 
$
\nbr{\widehat V_w^N - V^*}^2_{\mu\pi_b} \le 12(K+C_\infty)\epsilon_{\text{app}}^{\nu} + 2C_\nu(1 + \gamma)\epsilon_{\text{app}}^V(\lambda) + 6 C_\nu\epsilon_{\text{app}}^\pi(\lambda)  + 16\lambda^2 C^2_\pi+ \rbr{2\gamma^2 + 2}\rbr{\frac{\gamma\lambda}{1 - \gamma} H^*}^2 + 2\epsilon_{\text{stat}}(T)+2\nbr{\widehat V_w^N - \widehat V_w^*}^2_{\mu\pi_b},
 $
where $C_\infty = \max\cbr{\frac{C_R}{1 - \gamma} , C_\pi}$ and $C_\nu = \max_{\nu\in \Hcal_w} \nbr{\nu}_2$.
}
\end{theorem}
Detailed proof can be found in Appendix~\ref{appsec:error decomposition}. Ignoring the constant factors, the above results can be simplified as
$$
\textstyle
\nbr{\widehat V_w^N - V^*}^2_{\mu\pi_b} \leq \epsilon_{\text{app}}(\lambda)+ \epsilon_{\text{sm}}(\lambda)+\epsilon_{\text{stat}}(T) + \epsilon_{\text{opt}},
$$
where $\epsilon_{\text{app}}(\lambda): =\Ocal(\epsilon_{\text{app}}^{\nu} + \epsilon_{\text{app}}^{V}(\lambda) +\epsilon_{\text{app}}^\pi(\lambda) )$ corresponds to the approximation error,  $ \epsilon_{\text{sm}}(\lambda):= \Ocal(\lambda^2) $ corresponds to the bias induced by smoothing, and $\epsilon_{\text{stat}}(T):=\Ocal(1/\sqrt{T})$ corresponds to the statistical error.

There exists a delicate trade-off between the smoothing bias and approximation error. Using large $\lambda$ increases the smoothing bias but decreases the approximation error since the solution function space is better behaved. The concrete correspondence between $\lambda$ and $\epsilon_{\text{app}}(\lambda)$ depends on the specific form of the function approximators, which is beyond the scope of this paper. Finally, when the approximation is good enough (i.e., zero approximation error and full column rank of feature matrices), then our algorithm will converge to the optimal value function $V^*$ as $\lambda\to 0$, $N,T\to\infty$.

\section{Related Work}\label{sec:related_work}

One of our main contributions is a provably convergent algorithm when nonlinear approximation is used in the off-policy control case. Convergence guarantees exist in the literature for a few rather special cases, as reviewed in the introduction~\citep{Boyan95Generalization,Gordon95Stable,Tsitsiklis97Analysis,Ormoneit02Kernel, AntSzeMun08,melo08analysis}. Of particular interest is the Greedy-GQ algorithm~\citep{Maei10Toward}, who uses two time-scale analysis to shown asymptotic convergence only for \emph{linear} function approximation in the controlled case. However, it does not take the true gradient estimator in the algorithm, and the update of policy may become intractable when the action space is continuous.

Algorithmically, our method is most related to RL algorithms with entropy-regularized policies. Different from the motivation in our method where the entropy regularization is introduced in dual form for smoothing~\citep{Nesterov05}, the entropy-regularized MDP has been proposed for exploration~\citep{Defarias00,HaaTanAbbLev17}, taming noise in observations~\citep{RubShaTis12,FoxPakTis15}, and ensuring tractability~\citep{Todorov07}. Specifically, \citet{FoxPakTis15} proposed soft Q-learning for the tabular case, but its extension to the function approximation case is hard, as the summation operation in $\log$-sum-$\exp$ of the update rule becomes a computationally expensive integration. To avoid such a difficulty, \citet{HaaTanAbbLev17} approximate the integral via Monte-Carlo method with Stein variational gradient descent sampler, but limited theory is provided. Another related algorithm is developed by~\citet{AsaLit16} for the tabular case, which resembles SARSA with a particular policy; also see \citet{LiuRamLiuPen17} for a Bayesian variant. Observing the duality connection between soft Q-learning and maximum entropy policy optimization, \citet{NeuJonGom17} and \citet{SchAbbChe17} investigate the equivalence between these two types of algorithms.

Besides the difficulty to generalize these algorithms to multi-step trajectories in off-policy setting, the major drawback of these algorithms is the lack of theoretical guarantees when accompanying with function approximators. It is not clear whether the algorithms converge or not, do not even mention the quality of the stationary points. That said, \citet{NacNorXuSch17a,NacNorXuSch17b} also exploit the consistency condition in Theorem~\ref{thm:smoothed_bellman_opt_pi} and propose the PCL algorithm which optimizes the upper bound of the mean squared consistency Bellman error~\eq{eq:simplified_smooth_bellman_loss}. The same consistency condition is also discovered in~\citet{RawTouVij12}, and the proposed $\Phi$-learning algorithm can be viewed as fix-point iteration version of the the unified PCL with tabular $Q$-function. However, as we discussed in Section~\ref{sec:stoc_primal_dual}, the PCL algorithms becomes biased in stochastic environment, which may lead to inferior solutions~\citet{Baird95}.

Several recent works~\citep{ChenWang16,Wang17,Dai18boosting} have also considered saddle-point formulations of Bellman equations, but these formulations are fundamentally different from ours.  These saddle point problems are derived from the \emph{Lagrangian} dual of the linear programming formulation of Bellman equations~\citep{schweitzer85generalized,de2003linear}.  In contrast, our formulation is derived from the Bellman equation directly using \emph{Fenchel} duality/transformation.  It would be interesting to investigate the connection between these two saddle-point formulations in future work.

\section{Experiments}\label{sec:experiments}

The goal of our experimental evalation is two folds: {\bf i)}, to better understand of the effect of each algorithmic component in the proposed algorithm, and {\bf ii)}, to demonstrate the stability and efficiency of \algabb\ in both \emph{off-policy} and \emph{on-policy} settings. Therefore, we conducted an ablation study on~\algabb, and a comprehensive comparison to state-of-the-art reinforcement learning algorithms.  While we derive and present \algabb\ for single-step Bellman error case, it can be extended to multi-step cases (Appendix~\ref{appendix:multi_step}).  In our experiment, we used this multi-step version.

\subsection{Ablation Study}

To get a better understanding of the trade-off between the variance and bias, including both the bias from the smoothing technique and the introduction of the function approximator, we performed ablation study in the Swimmer-v1 environment with \emph{stochastic} transition by varying the coefficient for entropic regularization $\lambda$ and the coefficient of the dual function $\eta$ in the optimization~\eq{eq:one_step_trade_off}, as well as the number of the rollout steps, $k$.
\begin{figure*}[!t]
\centering
  \begin{tabular}{ccc}
    \includegraphics[width=0.3\textwidth,  trim={0.5cm 0.6cm 0.6cm 1.1cm},clip]{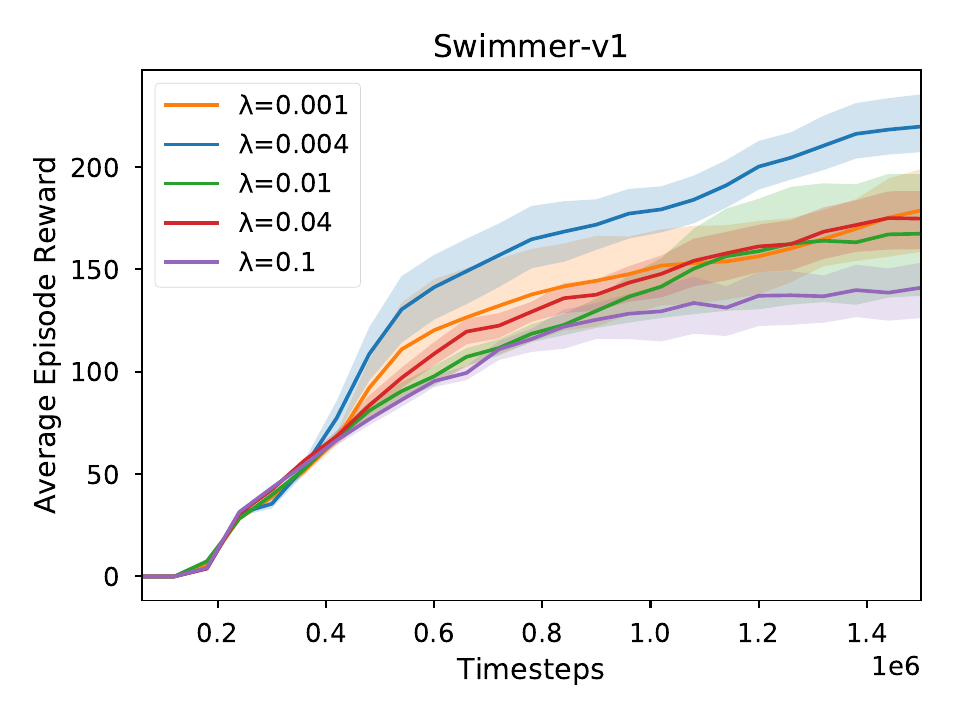}&
    \includegraphics[width=0.3\textwidth,  trim={0.5cm 0.6cm 0.6cm 1.1cm},clip]{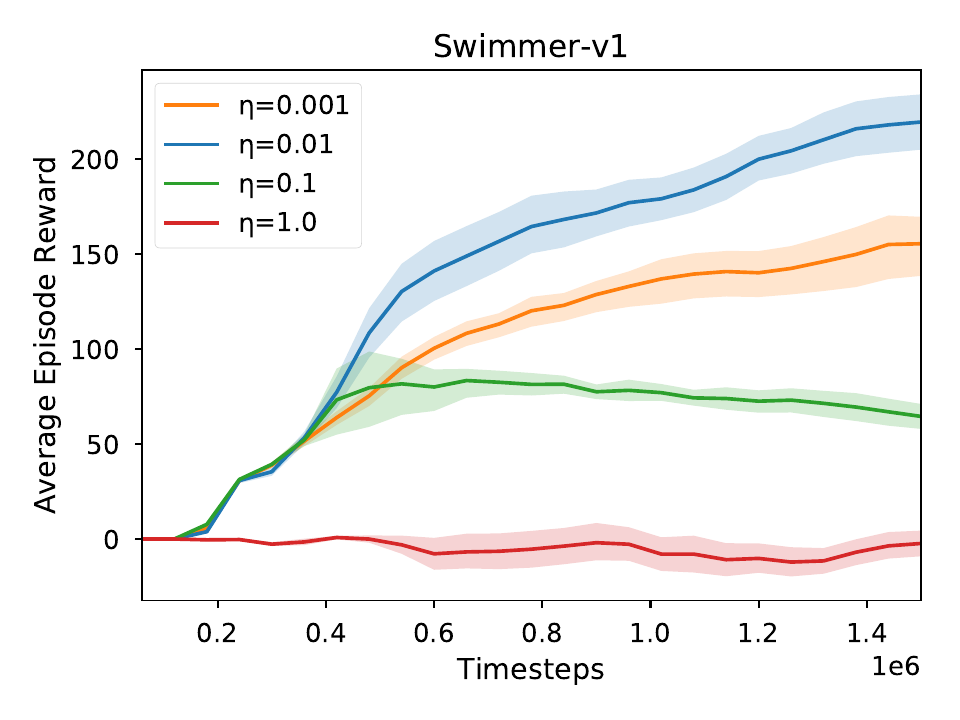}&
    \includegraphics[width=0.3\textwidth,  trim={0.5cm 0.6cm 0.6cm 1.1cm},clip]{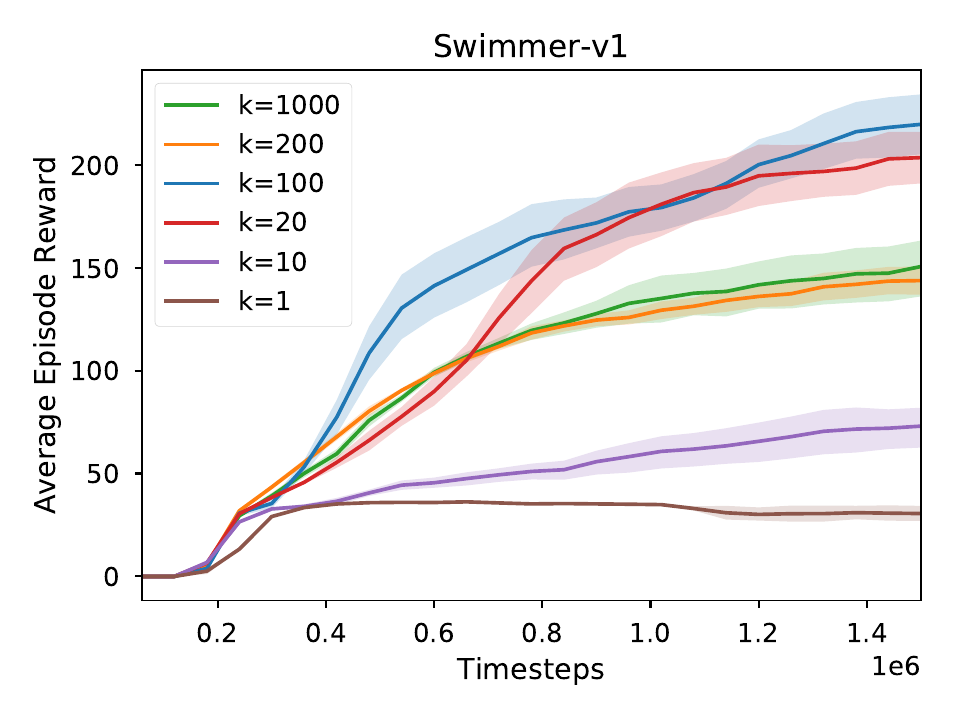}\\ 
    (a) $\lambda$ for entropy &(b) $\eta$ for dual  &(c) $k$ for steps\\
  \end{tabular}
  \caption{The ablation study of the~\algabb~on Swimmer-v1. We varied $\lambda$, $\eta$, and $k$ to justify the effect of each component in the algorithm.}
  \label{fig:ablation_study}
\end{figure*}
\begin{figure*}[!t]
\centering
  \begin{tabular}{cc}
    \includegraphics[width=0.3\textwidth,  trim={0.5cm 0.6cm 0.6cm 1.1cm},clip]{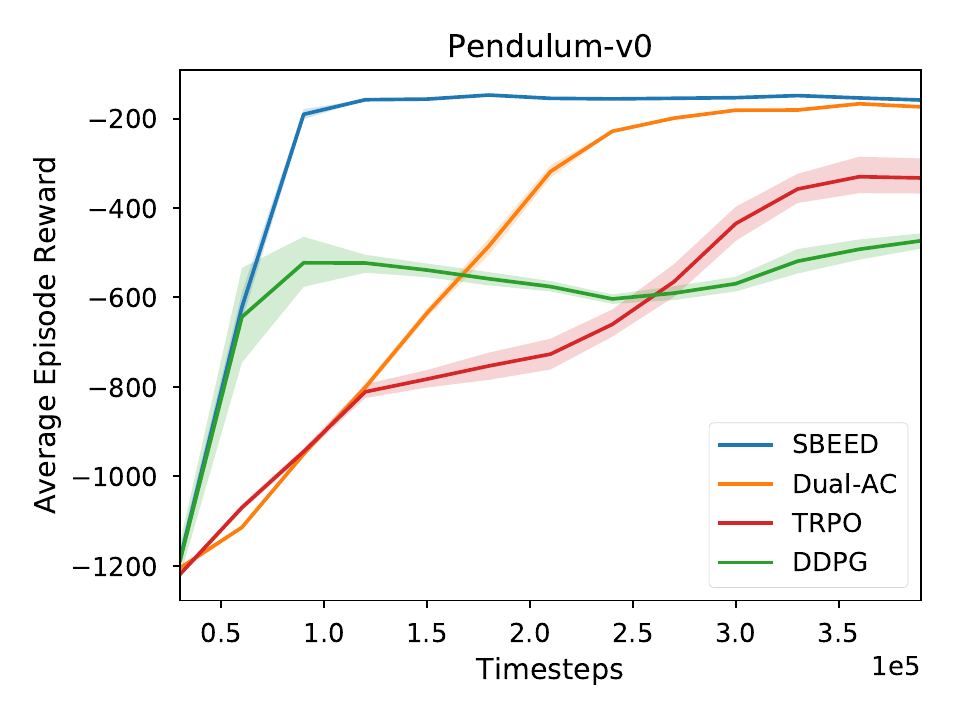}&
    \includegraphics[width=0.3\textwidth,  trim={0.5cm 0.6cm 0.6cm 1.1cm},clip]{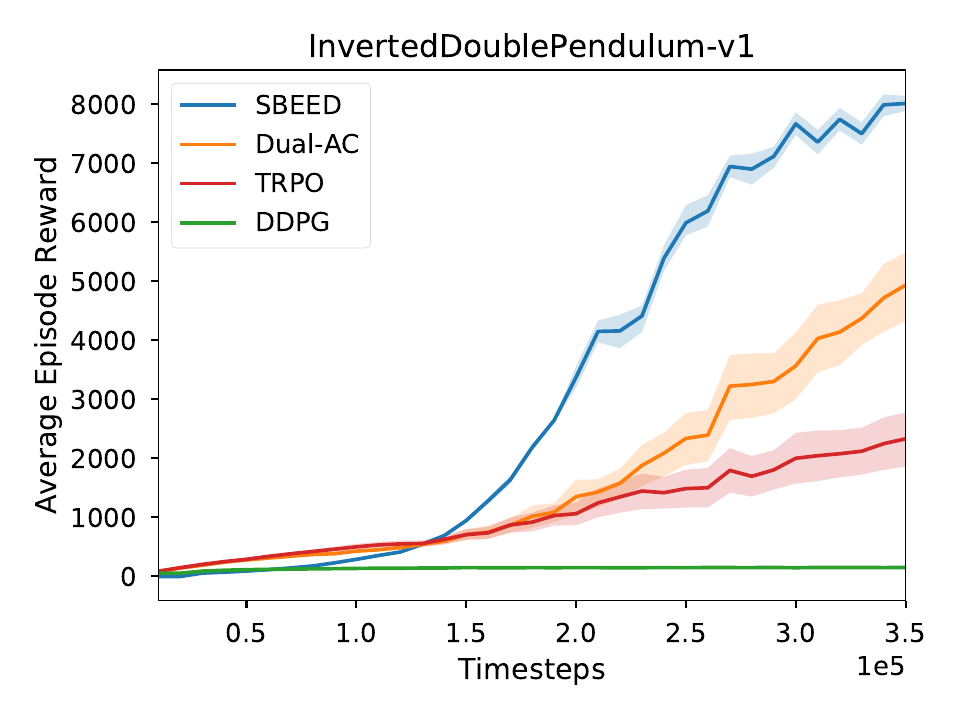}\\
    (a) Pendulum &(b) InvertedDoublePendulum \\
    \vspace{2mm}
  \end{tabular}
  \begin{tabular}{ccc}
    \includegraphics[width=0.3\textwidth,  trim={0.5cm 0.6cm 0.6cm 1.1cm},clip]{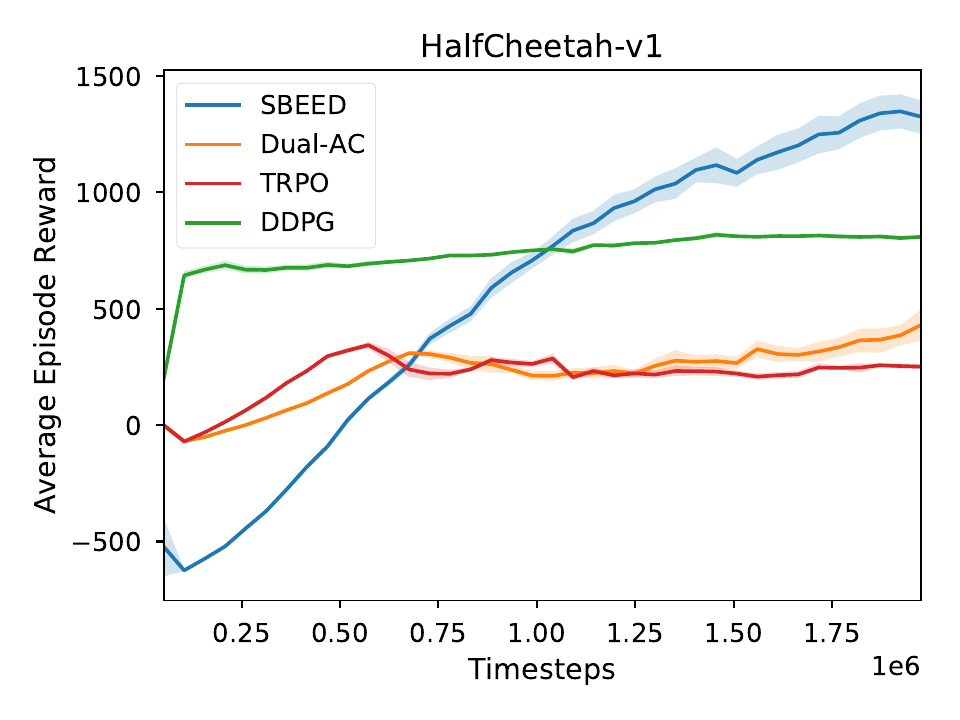}&
    \includegraphics[width=0.3\textwidth,  trim={0.5cm 0.6cm 0.6cm 1.1cm},clip]{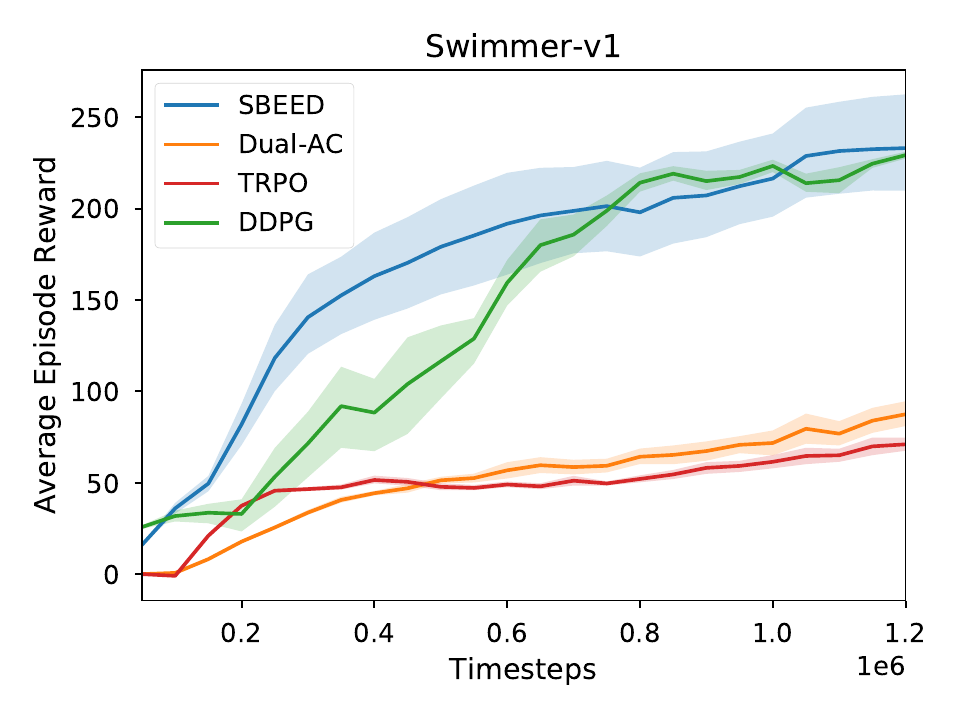}&
    \includegraphics[width=0.3\textwidth,  trim={0.5cm 0.6cm 0.6cm 1.1cm},clip]{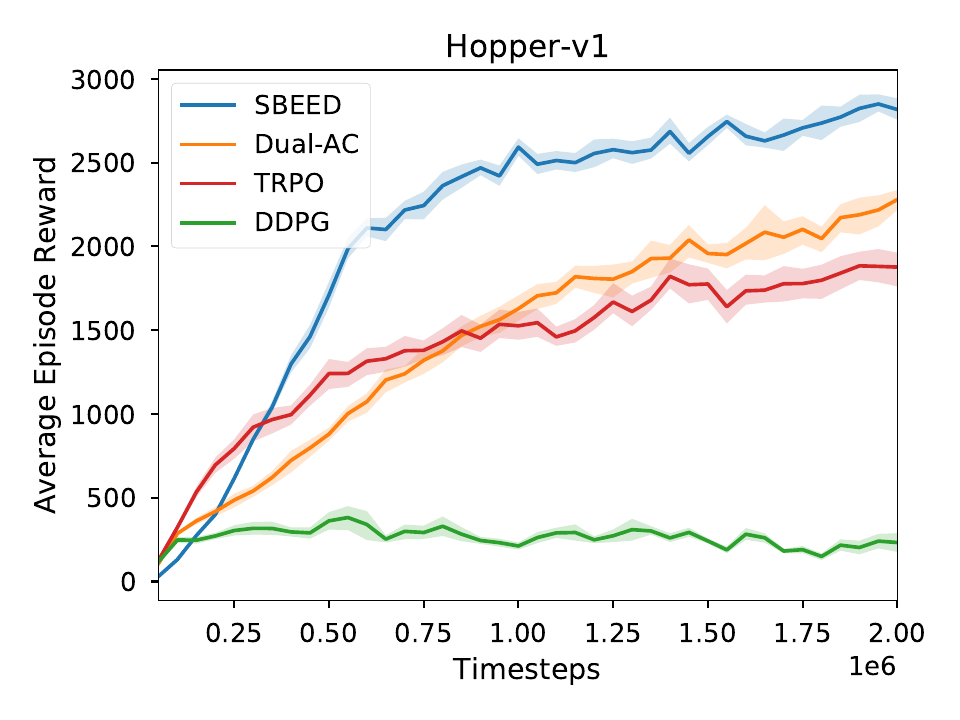}    \\ 
    (c) HalfCheetah &(d) Swimmer &(e) Hopper\\
  \end{tabular}
  \caption{The results of \algabb~against TRPO, Dual AC and DDPG. Each plot shows average reward during training across $5$ random runs, with $50\%$ confidence interval. The x-axis is the number of training iterations. \algabb~achieves significant better performance than the competitors on all tasks. 
  }
  \label{fig:offpolicy_comparison}
\end{figure*}

\paragraph{The effect of smoothing.} We introduced the entropy regularization to avoid non-smoothness in Bellman error objective. However, as we discussed, it also introduces bias. We varied $\lambda$ and evaluated the performance of \algabb. The results in Figure~\ref{fig:ablation_study}(a) are as expected: there is indeed an intermediate value for $\lambda$ that gives the best bias/smoothness balance.

\paragraph{The effect of dual function.} One of the important components in our algorithm is the dual function, which cancels the variance. The effect of such cancellation is controlled by $\eta\in[0, 1]$, and we expected an intermediate value gives the best performance. This is verified by the experiment of varying $\eta$, as shown in Figure~\ref{fig:ablation_study}(b). 

\paragraph{The effect of multi-step.} \algabb\ can be easily extended to the multi-step version. However, increasing the length of steps will also increase the variance. We tested the performance of the algorithm with different $k$ values. The results shown in Figure~\ref{fig:ablation_study}(c) confirms that an intermediate value for $k$ yields the best result.

\subsection{Comparison in Continuous Control Tasks}

We tested \algabb\ across multiple continuous control tasks from the OpenAI Gym benchmark~\citep{BroChePetSchetal16} using the MuJoCo simulator~\citep{TodEreTas12}, including Pendulum-v0, InvertedDoublePendulum-v1, HalfCheetah-v1, Swimmer-v1, and Hopper-v1. For fairness, we follows the default setting of the MuJoCo simulator in each task in this section. These tasks have dynamics of different natures, so are helpful for evaluating the behavior of the proposed~\algabb~in different scenarios. We compared \algabb\ with several state-of-the-art algorithms, including two on-policy algorithms, trust region policy optimization~(TRPO)~\citep{SchLevAbbJoretal15} dual actor-critic~(Dual AC)~\citep{Dai18boosting}, and one off-policy algorithm, deep deterministic policy gradient~(DDPG)~\citep{Lillicrap15Continuous}.  We did not include PCL~\citep{NacNorXuSch17a} as it is a special case of our algorithm by setting $\eta = 0$, \ie, ignoring the updates for dual function. Since TRPO and Dual-AC are only applicable for on-policy setting, for fairness, we also conducted the comparison with these two algorithm in on-policy setting. Due to the space limitation, the results can be found in Appendix~\ref{appendix:more_exp}.

We ran the algorithm with $5$ random seeds and reported the average rewards with 50\% confidence intervals. The results are shown in Figure~\ref{fig:offpolicy_comparison}. We can see that our \algabb\ algorithm achieves significantly better performance than all other algorithms across the board. These results suggest that the~\algabb~can exploit the off-policy samples efficiently and stably, and achieve a good trade-off between bias and variance.

It should be emphasized that the stability of algorithm is an important issue in reinforcement learning. As we can see from the results, although DDPG can also exploit the off-policy sample, which promotes its efficiency in stable environments, \eg, HalfCheetah-v1 and Swimmer-v1, it may fail to learn in unstable environments, \eg, InvertedDoublePendulum-v1 and Hopper-v1, which was observed by~\citet{HenIslBacPinetal17} and \citet{Haarnoja2018soft}. In contrast, \algabb\ is consistently reliable and effective in different tasks.

\section{Conclusion}\label{sec:conclusion}

We provide a new optimization perspective of the Bellman equation, based on which we develop the new \algabb\ algorithm for policy optimization in reinforcement learning. The algorithm is \emph{provably convergent} even when \emph{nonlinear} function approximation is used on \emph{off-policy} samples.  We also provide PAC bound to characterize the sample complexity based on \emph{one single off-policy sample path} collected by a fixed behavior policy.  Empirical study shows the proposed algorithm achieves superior performance across the board, compared to state-of-the-art baselines on several MuJoCo control tasks.

\section*{Acknowledgments}
Part of this work was done during BD's internship at Microsoft Research, Redmond. Part of the work was done when LL and JC were with Microsoft Research, Redmond. We thank Mohammad Ghavamzadeh, Csaba Szepesvari, and Greg Tucker for their insightful comments and discussions. NH is supported by NSF CCF-1755829. LS is supported in part by NSF IIS-1218749, NIH BIGDATA 1R01GM108341, NSF CAREER IIS-1350983, NSF IIS-1639792 EAGER, NSF CNS-1704701, ONR N00014-15-1-2340, Intel ISTC, NVIDIA and Amazon AWS.



\clearpage
\newpage

\onecolumn

\appendix

\begin{appendix}

\begin{center}
{\huge \textbf{Appendix}}
\end{center}

\section{Properties of Smoothed Bellman Operator}\label{appendix:alg_derivation_proof}

After applying the smoothing technique~\citep{Nesterov05}, we obtain a new Bellman operator, $\widetilde\Tcal$, which is contractive.  By such property, we can guarantee the uniqueness of the solution; a similar result is also presented in~\citet{FoxPakTis15,AsaLit16}.

\noindent\textbf{Proposition~\ref{thm:contraction}~(Contraction)}
\textit{$\Tcal_\lambda$ is a $\gamma$-contraction.  Consequently, the corresponding smoothed Bellman equation~\eq{eq:smoothed_bellman_opt_dual}, or equivalently \eq{eq:smoothed_bellman_opt_dual2}, has a unique solution $\Vtil^*$.}

\begin{proof}
For any $V_1, V_2:\Scal \to \RR$, we have
\begin{eqnarray*}
&&\nbr{\widetilde\Tcal V_1 - \widetilde\Tcal V_2}_\infty \\
&=& \nbr{\max_\pi \cbr{\inner{\pi}{R(s, a) + \gamma \EE_{s'|s, a}\sbr{V_1(s')}} + \lambda H(\pi)} - \max_\pi \cbr{\inner{\pi}{R(s, a) + \gamma \EE_{s'|s, a}\sbr{V_2(s')}} + \lambda H(\pi)}}_\infty\\
&\le& \nbr{\max_\pi \cbr{ \inner{\pi}{R(s, a) + \gamma \EE_{s'|s, a}\sbr{V_1(s')}} + \lambda H(\pi) - \inner{\pi}{R(s, a) + \gamma \EE_{s'|s, a}\sbr{V_2(s')}} - \lambda H(\pi)}} \\
&\le& \nbr{\max_\pi \inner{\pi}{\gamma \EE_{s'|s, a}\sbr{V_1(s') - V_2(s')}} }_\infty\\
&\le& \gamma \nbr{V_1 - V_2}_\infty\,.
\end{eqnarray*}
$\Tcal_\lambda$ is therefore a $\gamma$-contraction and, by the Banach fixed point theorem, admits a unique fixed point.
\end{proof}

Moreover, we may characterize the bias introduced by the entropic smoothing, similar to the simulation lemma (see, e.g., \citet{Kearns02Near} and \citet{Strehl09Reinforcement}):

\noindent\textbf{Proposition~\ref{lemma:smooth_bias}~(Smoothing bias)}
\textit{Let $V^*$ and $\Vtil^*$ be the fixed points of~\eq{eq:bellman_opt_dual} and~\eq{eq:smoothed_bellman_opt_dual}, respectively. It holds that
\begin{equation*}
\nbr{V^* - \Vtil^*}_\infty \le \frac{\lambda H^*}{1 - \gamma}.
\end{equation*}
As $\lambda\to 0$, $\Vtil^*$ converges to $V^*$ pointwisely.
}

\begin{proof}
Using the triangle inequality and the contraction property of $\Tcal_\lambda$, we have
\begin{eqnarray*}
\nbr{V^*-\Vtil^*}_\infty &=& \nbr{\Tcal V^* - \Tcal_\lambda \Vtil^*}_\infty \\
&=& \nbr{V^* - \Tcal_\lambda V^* + \Tcal_\lambda V^* - \Tcal_\lambda \Vtil^*}_\infty \\
&\le& \nbr{V^* - \Tcal_\lambda V^*}_\infty + \nbr{\Tcal_\lambda V^* - \Tcal_\lambda \Vtil^*}_\infty \\
&\le& \lambda H^* + \gamma \nbr{V^*-\Vtil^*}_\infty\,,
\end{eqnarray*}
which immediately implies the desired bound.
\end{proof}

The smoothed Bellman equation involves a $\log$-$\text{sum}$-$\exp$ operator to approximate the $\max$-operator, which increases the nonlinearity of the equation. We further characterize the solution of the smoothed Bellman equation, by the { temporal consistency} conditions. 

\noindent\textbf{Theorem~\ref{thm:smoothed_bellman_opt_pi}~(Temporal consistency)}
\textit{
Assume $\lambda>0$.  Let $\Vtil^*$ be the fixed point of~\eq{eq:smoothed_bellman_opt_dual} and  $\pi_\lambda^*$ the corresponding policy that attains the maximum on the RHS of~\eq{eq:smoothed_bellman_opt_dual}. Then, $(V,\pi)=(\Vtil^*,\pi_\lambda^*)$ if and only if $(V,\pi)$ satisfies the following equality for all $(s,a)\in\Scal\times\Acal$:
\begin{equation*}
V(s) = R(s, a) + \gamma\EE_{s'|s, a}\sbr{V(s')} - \lambda\log\pi(a|s) \,. \tag*{(7)}
\end{equation*}
}

\begin{proof} The proof has two parts. \\
\noindent\textbf{(Necessity)} 
We need to show $(\Vtil^*,\pi_\lambda^*)$ is a solution to \eqref{eq:smoothed_bellman_opt_pi}.  Simple calculations give the closed form of $\pi_\lambda^*$:
\[
\pi_\lambda^*(a|s) = Z(s)^{-1} \exp\rbr{\frac{R(s,a) + \gamma \EE_{s'|s,a}\sbr{\Vtil^*(s')}}{\lambda}}\,,
\]
where $Z(s) \defeq \sum_{a\in\Acal} \exp\rbr{\frac{R(s,a) + \gamma \EE_{s'|s,a}\sbr{\Vtil^*(s')}}{\lambda}}$ is a state-dependent normalization constant.  Therefore, for any $a\in\Acal$,
\begin{eqnarray*}
&& R(s, a) + \gamma\EE_{s'|s, a}\sbr{\Vtil^*(s')} - \lambda\log\pi_\lambda^*(a|s) \\
&=& R(s, a) + \gamma\EE_{s'|s, a}\sbr{\Vtil^*(s')} - \lambda \rbr{ \frac{R(s,a) + \gamma \EE_{s'|s,a}\sbr{\Vtil^*(s')}}{\lambda} - \log Z(s)} \\
&=& \lambda \log Z(s) \,\,=\,\, \Vtil^*(s)\,,
\end{eqnarray*}
where the last step is from \eqref{eq:smoothed_bellman_opt_dual2}.  Therefore, $(\Vtil^*, \pi_\lambda^*)$ satisfies \eqref{eq:smoothed_bellman_opt_pi}.

\noindent\textbf{(Sufficiency)} Assume $\bar V$ and $\bar\pi$ satisfies~\eq{eq:smoothed_bellman_opt_pi}, then we have for all $(s,a)\in\Scal\times\Acal$ that
\begin{eqnarray*}
\bar V(s) &=& R(s, a) + \gamma \EE_{s'|s, a}\sbr{\bar V(s')} - \lambda\log\bar \pi(a|s) \\
\pi(a|s) &=& \exp\rbr{\frac{R(s, a) + \gamma \EE_{s'|s, a}\sbr{\bar V(s')} - \Vbar(s)}{\lambda}} \,.
\end{eqnarray*}
Recall $\pi(\cdot|s)\in \Pcal$, we have
\begin{eqnarray*}
&&\sum_{a\in\Acal} \exp\rbr{\frac{R(s, a) + \gamma \EE_{s'|s, a}\sbr{\bar V(s')} - \Vbar(s)}{\lambda}} = 1\\
\Rightarrow && \sum_{a\in\Acal} \exp\rbr{\frac{R(s, a) + \gamma \EE_{s'|s, a}\sbr{\bar V(s')}}{\lambda}} = \exp\rbr{\frac{\Vbar(s)}{\lambda}} \\
\Rightarrow && \Vbar(s) = \lambda \log\rbr{\sum_{a\in\Acal}\exp\rbr{\frac{R(s, a) + \gamma \EE_{s'|s, a}\sbr{\Vbar(s')}}{\lambda}}} = \Tcal_\lambda \bar{V}(s)\,.
\end{eqnarray*}
The last equation holds for all $s\in\Scal$, so $\bar{V}$ is a fixed point of $\Tcal$.  It then follows from \propref{thm:contraction} that $\bar{V}=\Vtil^*$.  Finally, $\bar{\pi}=\pi_\lambda^*$ due to strong concavity of the entropy function
\end{proof}

The same conditions have been re-discovered several times, \eg,~\citep{RawTouVij12,NacNorXuSch17a}, from a completely different point of views.

\section{Variance Cancellation via the Saddle Point Formulation}\label{appendix:variance_reduciton}

The second term in the saddle point formulation~\eq{eq:variance_reduction} will cancel the variance $\VV_{s, a, s'}\sbr{\gamma V(s')}$.  Formally, 
\begin{proposition}\label{thm:variance_cancellation}
Given any fixed $(V,\pi)$, we have
\begin{eqnarray}
\max_{\rho\in\Fcal\rbr{\Scal\times \Acal}}-\EE_{s, a, s'}\sbr{\rbr{{R(s, a) + \gamma {V(s') - \lambda\log\pi(a|s)}}- \rho(s, a)}^2} = -\gamma^2 \EE_{s,a}\sbr{\VV_{s'|s,a}\sbr{V(s')}}.
\end{eqnarray}
\end{proposition}
\begin{proof}
Recall from \eqref{eq:variance_reduction} that $\delta(s,a,s') = R(s, a) + \gamma {V(s') - \lambda\log\pi(a|s)}$.  Then,
\begin{eqnarray*}
&&\max_{\rho}\,\,-\EE_{s, a, s'}\sbr{\rbr{{R(s, a) + \gamma {V(s') - \lambda\log\pi(a|s)}}- \rho(s, a)}^2} \\
&=& -\min_{\rho}\,\,\EE_{s,a}\sbr{\EE_{s'|s,a}\sbr{\rbr{\delta(s,a,s') - \rho(s,a)}^2}}\,.
\end{eqnarray*}
Clearly, the minimizing function $\rho^*$ may be determined for each $(s,a)$ entry separately.  Fix any $(s,a)\in\Scal\times\Acal$, and define a function on $\RR$ as $q(x) \defeq \EE_{s'|s,a}\sbr{\rbr{\delta(s,a,s') - x}^2}$.  Obviously, this convex function is minimized at the stationary point $x^*=\EE_{s'|s,a}\sbr{\delta(s,a,s')}$.  We therefore have $\rho^*(s,a) = \EE_{s'|s,a}\sbr{\delta(s,a,s')}$ for all $(s,a)$, so
\begin{eqnarray*}
& & \min_{\rho}\,\,\EE_{s,a}\sbr{\EE_{s'|s,a}\sbr{\rbr{\delta(s,a,s') - \rho(s,a)}^2}} \\
&=& \EE_{s,a}\sbr{\EE_{s'|s,a}\sbr{\rbr{\delta(s,a,s') - \EE_{s'|s,a}\sbr{\delta(s,a,s')}}^2}} \\
&=& \EE_{s,a}\sbr{\VV_{s'|s,a}\sbr{\delta(s,a,s')}} \\
&=& \EE_{s,a}\sbr{\VV_{s'|s,a}\sbr{\gamma V(s')}} \,\,=\,\, \gamma^2 \EE_{s,a}\sbr{\VV_{s'|s,a}\sbr{V(s')}} \,,
\end{eqnarray*}
where the second last step is due to the fact that, conditioned on $s$ and $a$, the only random variable in $\delta(s,a,s')$ is $V(s')$.
\end{proof}

\section{Details of \algabb}\label{appendix:algrithm_derivation}

In this section, we provide further details of the \algabb\ algorithms, including its gradient derivation and multi-step/eligibility-trace extension. 

\subsection{Unbiasedness of Gradient Estimator}

In this subsection, we compute the gradient with respect to the primal variables.  Let $(w_V, w_\pi)$ be the parameters of the primal $(V,\pi)$, and $w_\rho$ the parameters of the dual $\rho$.  Abusing notation a little bit, we now write the objective function $L_\eta(V,\pi;\rho)$ as $L_\eta(w_V, w_\pi; w_\rho)$.  Recall the quantity $\delta(s,a, s')$ from \eqref{eq:variance_reduction}.

\noindent \textbf{Theorem~\ref{thm:gradient_estimator} (Gradient derivation)}
\textit{Define $\bar{\ell}_\eta(w_V, w_\pi) \defeq L_\eta(w_V, w_\pi; w_\rho^*)$, where $w_\rho^* = \arg\max_{w_\rho} L_\eta(w_V, w_\pi; w_\rho)$.  Let $\dsas$ be a shorthand for $\delta(s,a,s')$, and $\hat{\rho}$ be dual parameterized by $w_\rho^*$.  Then,
\begin{align*}
\nabla_{w_V} \bar{\ell}_\eta =& 2 \EE_{s,a,s'}\sbr{\rbr{\dsas - V(s)}\rbr{\gamma { \nabla_{w_V} V(s')} - \nabla_{w_V} V(s)}} - 2 \eta\gamma\EE_{s,a,s'}\sbr{\rbr{\dsas - \hat{\rho}(s, a)}\nabla_{w_V} V(s')}\,, \\
\nabla_{w_\pi} \bar{\ell}_\eta =& -2\lambda\EE_{s, a, s'}\big[\rbr{(1 - \eta)\dsas + \eta\hat{\rho}(s, a) - V(s)} \cdot \nabla_{w_\pi}\log\pi(a|s)\big]\,. 
\end{align*}
}

\begin{proof}
First, note that $w_\rho^*$ is an implicit function of $(w_V,w_\pi)$.  Therefore, we must use the chain rule to compute the gradient:
\begin{eqnarray*} 
\nabla_{w_V} \bar{\ell}_\eta &=& 2\EE_{s,a,s'}\sbr{\rbr{\dsas -V(s;w_V)}\rbr{\gamma { \nabla_{w_V} V(s';w_V)} - \nabla_{w_V} V(s;w_V)}} \\
&& - 2\eta\gamma\EE_{s,a,s'}\sbr{\rbr{\dsas - \rho(s, a; w_\rho^*)} \nabla_{w_V}V(s';w_V)} \\
&& + 2\eta\gamma\EE_{s,a,s'}\sbr{\rbr{\dsas - \rho(s, a; w_\rho^*)}\nabla_{w_V}\rho(s,a;w_\rho^*)}\,.
\end{eqnarray*}
We next show that the last term is zero:
\begin{eqnarray*}
&& \EE_{s,a,s'}\sbr{\rbr{\dsas - \rho(s, a; w_\rho^*)}\nabla_{w_V}\rho(s,a;w_\rho^*)} \\
&=& \EE_{s,a,s'}\sbr{\rbr{\dsas - \rho(s, a; w_\rho^*)} \cdot \nabla_{w_V}w_\rho^* \cdot \nabla_{w_\rho}\rho(s,a;w_\rho^*)} \\
&=& \nabla_{w_V}w_\rho^* \cdot \EE_{s,a,s'}\sbr{\rbr{\dsas - \rho(s, a; w_\rho^*)} \cdot \nabla_{w_\rho}\rho(s,a;w_\rho^*)} \\
&=& \nabla_{w_V}w_\rho^* \cdot \mathbf{0} \,\,=\,\, \mathbf{0}\,,
\end{eqnarray*}
where the first step is the chain rule; the second is due to the fact that $\nabla_{w_V}w_\rho^*$ is not a function of $(s,a,s')$, so can be moved outside of the expectation; the third step is due to the optimality of $w_\rho^*$.  The gradient w.r.t. $w_V$ is thus derived.  The case for $w_\pi$ is similar.
\end{proof}

\subsection{Multi-step Extension}\label{appendix:multi_step}

One way to interpret the smoothed Bellman equation \eqref{eq:smoothed_bellman_opt_dual} is to treat each $\pi(\cdot|s)$ as a (mixture) action; in other words, the action space is now the simplex $\Pacal$.  With this interpretation, the introduced entropy regularization may be viewed as a shaping reward: given a mixture action $\pi(\cdot|s)$, its immediate reward is given by
\[
\tilde{R}(s,\pi(\cdot|s)) \defeq \EE_{a \sim \pi(\cdot|s)} \sbr{R(s,a)} + \lambda H(\pi,s) \,.
\]
The transition probabilities can also be adapted accordingly as follows
\[
\tilde{P}(s'|s,\pi(\cdot|s)) \defeq \EE_{a\in\pi(\cdot|s)}\sbr{P(s'|s,a)} \,.
\]
It can be verified that the above constructions induce a well-defined MDP $\tilde{M} = \langle \Scal, \Pacal, \tilde{P}, \tilde{R}, \gamma \rangle$, whose standard Bellman equation is exactly \eqref{eq:smoothed_bellman_opt_dual}.

With this interpretation, the proposed framework and algorithm can be easily applied to multi-step and eligibility-traces extensions.  Specifically, one can show that $(\Vtil^*,\pi_\lambda^*)$ is the unique solution that satisfies the multi-step expansion of \eq{eq:smoothed_bellman_opt_pi}: for any $k \ge 1$ and any $(s_0,a_0,a_1,\ldots,a_{k-1}) \in \Scal \times \Acal^k$,
\begin{eqnarray}\label{eq:multi_step_smooth_bellman}
V(s_0)= \sum_{t=0}^{k-1} \gamma^t \EE_{s_t|s_0, a_{0:t-1}}\sbr{R(s_t, a_t) - \lambda \log\pi(a_t|s_t)} + \gamma^k \EE_{s_k|s_0, a_{0:k-1}}\sbr{V(s_k)} \,.
\end{eqnarray}
Clearly, when $k=1$ (the single-step bootstrapping case), the above equation reduces to \eq{eq:smoothed_bellman_opt_pi}.

The $k$-step extension of objective function \eqref{eq:simplified_smooth_bellman_loss}  now becomes
\begin{equation*}
\min_{V, \pi} \EE_{s_0, a_{0:k-1}}\sbr{\rbr{\sum_{t=0}^{k-1} \gamma^t \EE_{s_t|s_0, a_{0:t-1}}\sbr{R(s_t, a_t) - \lambda \log\pi(a_t|s_t)} + \gamma^k \EE_{s_k|s_0,a_{0:k-1}}\sbr{V(s_k)}- V(s_0)}^2 }.
\end{equation*}
Applying the Legendre-Fenchel transformation and the interchangeability principle, we arrive at the following multi-step primal-dual optimization problem:
\begin{eqnarray*}
\min_{V, \pi}\max_\nu&& \EE_{s_0, a_{0:t-1}}\bigg[\nu(s_0, a_{0:t-1})\bigg(\sum_{t=0}^{k-1} \gamma^t \EE_{s_t|s_0, a_{0:k-1}}\sbr{R(s_t, a_t) - \lambda \log\pi(a_t|s_t)} \\
&&+ \gamma^k \EE_{s_k|s_0, a_{0:k-1}}\sbr{V(s_k)}- V(s_0)\bigg) \bigg]- \frac{1}{2}\EE_{s_0, a_{0:k-1}}\sbr{\nu(s_0, a_{0:k-1})^2}\\
=\min_{V, \pi}\max_\nu&& \EE_{s_{0:k},a_{0:k-1}}\bigg[\nu(s_0, a_{0:k-1}) \bigg(\sum_{t=0}^{k-1} \gamma^t \rbr{{R(s_t, a_t) - \lambda \log\pi(a_t|s_t)}} \\
&&+ \gamma^k{V(s_k)}- V(s_0)\bigg) \bigg] - \frac{1}{2}\EE_{s_0, a_{0:k-1}}\sbr{\nu(s_0, a_{0:k-1})^2}.
\end{eqnarray*}

Similar to the single-step case, defining
\[
\delta(s_{0:k}, a_{0:k-1}) \defeq \sum_{t=0}^{k-1} \gamma^t \rbr{R(s_t, a_t) - \lambda\log \pi(a_{t}|s_{t})} + \gamma^k{V(s_k)}\,.
\]
and using the substitution $\rho(s_0,a_{0:k-1}) = \nu(s_0,a_{0:k-1})+V(s_0)$, we reach the following saddle-point formulation:
\begin{equation}\label{eq:variance_reduction_kstep}
\min_{V,\pi}\max_\rho L(V,\pi;\rho) \defeq \EE_{s_{0:k}, a_{0:k-1}}\sbr{\rbr{\delta(s_{0:k}, a_{0:k-1})- V(s_0)}^2-\eta\rbr{{\delta(s_{0:k}, a_{0:k-1})}- \rho(s_0, a_{0:k-1})}^2}
\end{equation}
where the dual function now is $\rho(s_0, a_{0:k-1})$, a function on $\Scal\times \Acal^k$, and $\eta\geq 0$ is again a parameter used to balance between bias and variance. It is straightforward to generalize Theorem~\ref{thm:gradient_estimator} to the multi-step setting, and to adapt \algabb\ accordingly, 

\subsection{Eligibility-trace Extension}\label{appendix:eligibility_trace}

Eligibility traces can be viewed as an approach to aggregating multi-step bootstraps for $k\in\{1,2,\cdots\}$; see \citet{SutBar98} for more discussions.  The same can be applied to the multi-step consistency condition \eqref{eq:multi_step_smooth_bellman}, using an exponential weighting parameterized by $\zeta\in[0,1)$.  Specifically, for all $(s_0,a_{0:k-1})\in\Scal\times\Acal^k$, we have
\begin{eqnarray}\label{eq:trace_smooth_bellman}
V(s_0) = (1 - \zeta)\sum_{k=1}^\infty\zeta^{k-1}\rbr{\sum_{t=0}^{k-1} \gamma^t \EE_{s_t|s_0, a_{0:k-1}} \sbr{R(s_t, a_t) - \lambda \log\pi(a_t|s_t)} + \gamma^k \EE_{s_k|s_0, a_{0:k-1}}\sbr{V(s_k)}} \,.
\end{eqnarray} 
Then, following similar steps as in the previous subsection, we reach the following saddle-point optimization:
\begin{eqnarray}
\min_{V, \pi}\max_\rho && \EE_{s_{0:\infty},a_{0:\infty}} \sbr{\rbr{(1 - \zeta)\sum_{k=1}^\infty\zeta^{k-1}\delta(s_{0:k}, a_{0:k-1})- V(s_0)}^2} \nonumber \\
&& - \eta \, \EE_{s_{0:\infty},a_{0:\infty}}\sbr{\rbr{{(1 - \zeta)\sum_{k=1}^\infty\zeta^{k-1} \delta(s_{0:k}, a_{0:k-1})} - \rho(s_0, a_{0:\infty})}^2} \,. \label{eq:dual_eligible_trace}
\end{eqnarray}
In practice, $\rho(s_0, a_{0:\infty})$ can be parametrized by neural networks with finite length of actions as input as an approximation. 

\section{Proof Details of the Theoretical Analysis}\label{appendix:analysis_details}

In this section, we provide the details of the analysis in Theorems~\ref{thm:statistical_error} and~\ref{thm:error_decomposition}.  We start with the boundedness of $V^*$ and $\Vtil^*$ under \asmpref{asmp:mdp_reg}.  Given any measure on the state space $\Scal$,
\[
\nbr{V^*}_\mu \le \nbr{V^*}_\infty \le (1+\gamma+\gamma^2+\cdots) C_R = C_V \defeq \frac{C_R}{1-\gamma}\,.
\]
A similar argument may be used on $\Vtil^*$ to get
\[
\nbr{\Vtil^*}_\mu \le \frac{C_R+H^*}{1-\gamma}\,.
\]

It should be emphasized that although \asmpref{asmp:mdp_reg} ensures boundedness of $V^*$ and $\log \pi^*(a|s)$, it does not imply the continuity and smoothness.  In fact, as we will see later, $\lambda$ controls the trade-off between approximation error (due to parameterization) and bias (due to smoothing) in the solution of the smoothed Bellman equation.

\subsection{Error Decomposition} \label{appsec:error decomposition}
Recall that 
\begin{itemize}
\item $(V^*,\pi^*)$ corresponds to the optimal value function and optimal policy to the original Bellman equation, namely, they are solutions to the optimization problem (\ref{eq:mean_square_bellman_opt});
\item $(V_\lambda^*,\pi_\lambda^*)$ corresponds to the optimal value function and optimal policy to the smoothed Bellman equation, namely, they are solutions to the optimization problem (\ref{eq:simplified_smooth_bellman_loss}) with objective $\ell(V,\pi)$;
\item  $(V^*_w,\pi^*_w)$ correponds to the optimal solution to the optimization problem  (\ref{eq:simplified_smooth_bellman_loss}) under nonlinear function approximation, with objective $\ell_w (V_w,\pi_w)$;
\item $(\widehat V^*_w,\widehat\pi^*_w)$ stands for the optimal solution to the finite sample approximation of (\ref{eq:simplified_smooth_bellman_loss}) under nonlinear function approximation, with objective $\widehat\ell_T (V_w,\pi_w)$. 
\end{itemize} 

Hence, we can  decompose the error between  $(\widehat V^*_w,\widehat\pi^*_w)$ and $(V_\lambda^*,\pi_\lambda^*)$ under the $\nbr{\cdot}_{\mu\pi_b}$ norm.

\begin{eqnarray}\label{eq:decompose_error}
\nbr{\widehat V_w^* - V^*}_{\mu\pi_b}^2 &\le& 2\nbr{\widehat V_w^* -  V_\lambda^*}_{\mu\pi_b}^2 + 2\nbr{\Vtil^* - V^*}^2_{\mu\pi_b}.
\end{eqnarray}

We first look at the second term from smoothing error, which can be similarly bounded, as shown in Proposition~\ref{lemma:smooth_bias}. 
\begin{lemma}[Smoothing bias]\label{lemma:decomposed_bias}
$\nbr{\Vtil^* - V^*}^2_{\mu\pi_b} \le (2\gamma^2 + 2)\rbr{\frac{\gamma\lambda}{1 - \gamma} \max_{\pi\in\Pcal} H(\pi)}^2$.
\end{lemma}
\begin{proof}
For $\nbr{\Vtil^* - V^*}^2_{\mu\pi_b}$, we have
\begin{eqnarray*}
\nbr{\Vtil^* - V^*}^2_{\mu\pi_b} &=& \int \rbr{\gamma \EE_{s'|s, a}\sbr{V^*(s') - \Vtil^*(s')} - \rbr{V^*(s) - \Vtil^*(s)}}^2\mu(s)\pi_b(a|s)dsda\\
&\le& 2\gamma^2 \nbr{\EE_{s'|s, a}\sbr{V^*(s') - \Vtil^*(s')}}_\infty^2 + 2\nbr{V^*(s) - \Vtil^*(s)}_\infty \\
&\le& (2\gamma^2 + 2)\rbr{\frac{\gamma\lambda}{1 - \gamma} \max_{\pi\in\Pcal} H(\pi)}^2,
\end{eqnarray*}
where the final inequality is because Lemma~\ref{lemma:smooth_bias}.
\end{proof}

We now look at the first term and show that 
\begin{lemma}\label{lemma:loss_solution_connection}
\begin{eqnarray*}
\nbr{\widehat V_w^* - V_\lambda^*}^2_{\mu\pi_b}&\le&2\rbr{\ell(\widehat V_w^*, \widehat \pi_w^*) - \ell(V_\lambda^*, \pi_\lambda^*) } + 4\lambda^2\nbr{\log{\widehat\pi_w^*(a|s)} - \log{\hat\pi_w^*(a|s)}}_2^2 \\
&&+ 4\lambda^2\nbr{\log{\pi_w^*(a|s)} - \log{\pi_\lambda^*(a|s)}}_2^2.
\end{eqnarray*}
\end{lemma}
\begin{proof}
Specifically, due to the strongly convexity of square function, we have
\begin{eqnarray*}
\ell(\widehat V_w^*, \widehat \pi_w^*) - \ell(V_\lambda^*, \pi_\lambda^*) &=&2\EE\sbr{\bar \Delta_{V_\lambda^*, \pi_\lambda^*}(s, a)\rbr{\bar \Delta_{\widehat V_w^*, \widehat \pi_w^*}(s, a) - \bar \Delta_{V_\lambda^*, \pi_\lambda^*}(s, a)}} \\
&+& \EE_{\mu\pi_b}\sbr{\rbr{\bar \Delta_{\widehat V_w^*, \widehat \pi_w^*}(s, a) - \bar \Delta_{V_\lambda^*, \pi_\lambda^*}(s, a)}^2} \\
&\ge& \int \rbr{\bar \Delta_{\widehat V_w^*, \widehat \pi_w^*}(s, a) - \bar \Delta_{V_\lambda^*, \pi_\lambda^*}(s, a)}^2 \mu(s)\pi_b(a|s)dsda\\
&:=& \nbr{\bar \Delta_{\widehat V_w^*, \widehat \pi_w^*}(s, a) - \bar \Delta_{V_\lambda^*, \pi_\lambda^*}(s, a)}_2^2,
\end{eqnarray*}
where $\Delta(s, a, s') = {R(s, a) + \gamma {V(s') - \lambda\log\pi(a|s)}- V(s)}$ and the second inequality is because the optimality of $V_\lambda^*$ and $\pi_\lambda^*$. Therefore, we have
\begin{eqnarray*}
&&\sqrt{\ell(\widehat V_w^*, \widehat \pi_w^*) - \ell(V_\lambda^*, \pi_\lambda^*) } \ge \nbr{\bar \Delta_{\widehat V_w^*, \widehat \pi_w^*}(s, a) - \bar \Delta_{V_\lambda^*, \pi_\lambda^*}(s, a)}_2 \\
&\ge&\abr{\nbr{\gamma \EE_{s'|s, a}\sbr{\widehat V_w^*(s') - V_\lambda^*(s')} - \rbr{\widehat V_w^*(s) - V_\lambda^*(s)}}_2- \lambda\nbr{\log{\widehat \pi_w^*(a|s)} - \log{\pi_\lambda^*(a|s)}}_2}\\
&=&\abr{\nbr{\widehat V_w^* - V_\lambda^*}_{\mu\pi_b} - \lambda\nbr{\log{\widehat \pi_w^*(a|s)} - \log{\pi_\lambda^*(a|s)}}_2}
\end{eqnarray*}
which implies
\begin{eqnarray*}
\nbr{\widehat V_w^* - V_\lambda^*}^2_{\mu\pi_b}&\le& 2\rbr{\ell(\widehat V_w^*, \widehat \pi_w^*) - \ell(V_\lambda^*, \pi_\lambda^*) } + 2\lambda^2\nbr{\log{\widehat \pi_w^*(a|s)} - \log{\pi_\lambda^*(a|s)}}_2^2\\
&\le&2\rbr{\ell(\widehat V_w^*, \widehat \pi_w^*) - \ell(V_\lambda^*, \pi_\lambda^*) } + 4\lambda^2\nbr{\log{\widehat\pi_w^*(a|s)} - \log{\pi_w^*(a|s)}}_2^2 \\
&&+ 4\lambda^2\nbr{\log{ \pi_w^*(a|s)} - \log{\pi_\lambda^*(a|s)}}_2^2.
\end{eqnarray*}
\end{proof}

In regular MDP with Assumption~\ref{asmp:mdp_reg}, with appropriate $C$, such constraint does not introduce any loss. We denote the family of value functions and policies by parametrization as $\Vcal_w,\, \Pcal_w$, respectively. Then, for $V$ and $\log\pi$ uniformly bounded by $C_\infty = \max\cbr{\frac{C_R}{1 - \gamma} , C_\pi}$ and the square loss is uniformly $K$-Lipschitz continuous, by proposition in~\citet{DaiHePanBooetal16}, we have
\begin{corollary}\label{cor:dual_approx_error}
$\ell(V,\pi) - \ell_w(V, \pi)\le (K+C_\infty)\epsilon_{\text{app}}^\nu$ where $\epsilon_{\text{app}} = \sup_{\nu\in \Ccal}\inf_{h\in\Hcal}\nbr{\nu - h}_\infty$ with $\Ccal$ denoting the Lipschitz continuous function space and $\Hcal$ denoting the hypothesis space.
\end{corollary}
\begin{proof}
Denote the $\phi(V, \pi, \nu) := \EE_{s, a, s'}\sbr{\nu(s, a)\rbr{R(s, a) + \gamma {V(s') - \lambda\log\pi(a|s)}- V(s)}} - \frac{1}{2}\EE_{s, a, s'}\sbr{\nu^2(s, a)}$, we have $\phi(V, \pi, \nu)$ is $(K+C_\infty)$-Lipschitz continuous w.r.t. $\nbr{\cdot}_\infty$. Denote $\nu_{V, \pi}^* = \argmax_\nu\phi(V, \pi, \nu)$, $\nu_{V, \pi}^\Hcal = \argmax_{\nu\in\Hcal}\phi(V, \pi, \nu)$, and $\hat \nu_{V, \pi} = \min_{\nu\in\Hcal}\nbr{\nu - \nu_{V, \pi}^*}_\infty$
\begin{eqnarray*}
\ell(V, \pi) - \ell_w(V, \pi) &=& \phi(V, \pi, \nu_{V, \pi}^* ) - \phi(V, \pi, \nu_{V, \pi}^\Hcal) \\
&\le&  \phi(V, \pi, \nu_{V, \pi}^* ) - \phi(V, \pi, \hat\nu_{V, \pi})\le (K+C_\infty)\epsilon_{\text{app}}^\nu.
\end{eqnarray*}

\end{proof}

For the third term in Lemma~\ref{lemma:loss_solution_connection}, we have
\begin{eqnarray}\label{eq:first_term}
\lambda\nbr{\log{\pi_w^*(a|s)} - \log{\pi_\lambda^*(a|s)}}^2_2&\le& {\ell\rbr{V, \pi_w^*} - \ell\rbr{V, \pi_\lambda^*} } \\
&=& \ell_w\rbr{V, \pi_w^*} - \ell_w\rbr{V, \pi_\lambda^*} + \rbr{\ell\rbr{V, \pi_w^*}- \ell_w\rbr{V, \pi_w^*} }- \rbr{\ell\rbr{V, \pi_\lambda^*}- \ell_w\rbr{V, \pi_\lambda^*}} \nonumber\\
&\le& C_\nu \inf_{\pi_w} \nbr{\lambda\log\pi_w - \lambda\log\pi_\lambda^*}_\infty +  2(K+C_\infty)\epsilon_{\text{app}}^\nu \nonumber \\
&\le& C_\nu \epsilon_{\text{app}}^\pi(\lambda)  +  2(K+C_\infty)\epsilon_{\text{app}}^\nu\nonumber
\end{eqnarray}
where $C_\nu = \max_{\nu\in \Hcal_w} \nbr{\nu}_2$. The first inequality comes from the strongly convexity of $\ell\rbr{V, \pi}$ w.r.t. $\lambda\log\pi$, the second inequality comes from Section 5 in~\citet{Bach14} and~\corref{cor:dual_approx_error} with $\epsilon_{\text{app}}^\pi(\lambda) \defeq \sup_{\pi\in\Pcal_\lambda}\inf_{\pi_w\in \Pcal_w}\nbr{\lambda\log\pi_w - \lambda \log \pi}_\infty$ with 
$$
\Pcal_\lambda := \cbr{\pi\in\Pcal, \pi(a|s) = \exp\rbr{\frac{Q(s, a) - \Lcal(Q)}{\lambda}}, \nbr{Q}_2\le C_V}.
$$
Based on the derivation of $\Pcal_\lambda$, with continuous $\Acal$, it can be seen that as $\lambda\rightarrow 0$, 
$$
\Pcal_0 = \cbr{\pi\in \Pcal, \pi(a|s) = \delta_{a_{\max}(s)}(a)},
$$ 
which results $\epsilon_{\text{app}}^\pi(\lambda)\rightarrow\infty$, and as $\lambda$ increasing as finite, the policy becomes smoother, resulting smaller approximate error in general. With discrete $\Acal$, although the $\epsilon_{\text{app}}^\pi(0)$ is bounded, the approximate error still decreases as $\lambda$ increases. The similar correspondence also applies to $\epsilon_{\text{app}}^V(\lambda)$. The concrete correspondence between $\lambda$ and $\epsilon_{\text{app}}(\lambda)$ depends on the specific form of the function approximators, which is an open problem and out of the scope of this paper. 

For the second term in~\ref{lemma:loss_solution_connection}, 
\begin{equation}\label{eq:second_term}
\lambda\nbr{\log{\widehat\pi_w^*(a|s)} - \log{\pi_w^*(a|s)}}_2\le \lambda\nbr{\log{\widehat\pi_w^*(a|s)}}_2 + \lambda\nbr{\log{\pi_w^*(a|s)}}_2\le 2\lambda C_\pi.
\end{equation} 

For the first term, we have 
\begin{eqnarray}\label{eq:third_term}
&&\ell(\widehat V_w^*, \widehat \pi_w^*) - \ell(V_\lambda^*, \pi_\lambda^*) \\
&=& \ell(\widehat V_w^*, \widehat \pi_w^*) - \ell_w(\widehat V_w^*, \widehat \pi_w^*) + \ell_w(\widehat V_w^*, \widehat \pi_w^*) - \ell_w(V_\lambda^*, \pi_\lambda^*) + \ell_w(V_\lambda^*, \pi_\lambda^*)- \ell(V_\lambda^*, \pi_\lambda^*) \nonumber \\
&\le& 2(K+C_\infty)\epsilon_{\text{app}}^{\nu} + \ell_w(\widehat V_w^*, \widehat \pi_w^*) - \ell_w(V_\lambda^*, \pi_\lambda^*)\nonumber \\
&=&2(K+C_\infty)\epsilon_{\text{app}}^{\nu} + \ell_w(\widehat V_w^*, \widehat \pi_w^*) - \ell_w(V_w^*, \pi_w^*) + {\ell_w(V_w^*, \pi_w^*) - \ell_w(V_\lambda^*, \pi_\lambda^*)}\nonumber \\
&\le&2(K+C_\infty)\epsilon_{\text{app}}^{\nu} + C_\nu \rbr{(1 + \gamma)\epsilon_{\text{app}}^V(\lambda) + \epsilon_{\text{app}}^\pi(\lambda)}+ \ell_w(\widehat V_w^*, \widehat \pi_w^*) - \ell_w( V_w^*, \pi_w^*). \nonumber
\end{eqnarray} 
The last inequality is because
\begin{eqnarray*}
\ell_w(V_w^*, \pi_w^*) - \ell_w(V_\lambda^*, \pi_\lambda^*) &=& \inf_{V_w, \pi_w}\ell_w(V_w, \pi_w) - \ell_w(V_\lambda^*, \pi_\lambda^*) \\
&\le& C_\nu\inf_{V_w, \pi_w}\rbr{\rbr{1 + \gamma}\nbr{V_w - V_\lambda^*}_\infty + \lambda\nbr{\log \pi_w - \log\pi_\lambda^*}_\infty}\\
&\le& C_\nu \rbr{(1 + \gamma)\epsilon_{\text{app}}^V(\lambda) + \epsilon_{\text{app}}^\pi(\lambda)},
\end{eqnarray*}
where the second inequality comes from Section 5 in~\citet{Bach14}.

Combine~\eq{eq:first_term},~\eq{eq:second_term} and~\eq{eq:third_term} into Lemma~\ref{lemma:loss_solution_connection} and Lemma~\ref{lemma:decomposed_bias} together with~\eq{eq:decompose_error}, we achieve
\begin{lemma}[Error decomposition]\label{lemma:error_decomposition}
{
\begin{eqnarray*}
\nbr{\widehat V_w^* - V^*}^2_{\mu\pi_b} &\le& \underbrace{2\rbr{4(K+C_\infty)\epsilon_{\text{app}}^{\nu} + C_\nu(1 + \gamma)\epsilon_{\text{app}}^V(\lambda) + 3C_\nu\epsilon_{\text{app}}^\pi(\lambda) } }_{\text{approximation error due to parametrization}}\\
&& + \underbrace{16\lambda^2 C^2_\pi+ \rbr{2\gamma^2 + 2}\rbr{\frac{\gamma\lambda}{1 - \gamma} \max_{\pi\in\Pcal} H(\pi)}^2 }_{\text{bias due to smoothing}}+ \underbrace{2\rbr{\ell_w(\widehat V_w^*, \widehat \pi_w^*) - \ell_w( V_w^*, \pi_w^*) }}_{\text{statistical error}}.
\end{eqnarray*}
}
\end{lemma}
We can see that the bound includes the errors from three aspects: {\bf i)}, the approximation error induced by parametrization of $V$, $\pi$, and $\nu$; {\bf ii)}, the bias induced by smoothing technique; {\bf iii)}, the statistical error. As we can see from Lemma~\ref{lemma:error_decomposition}, $\lambda$ plays an important role in balance the approximation error and smoothing bias.

\subsection{Statistical Error} \label{appsec:statistical error}

In this section, we analyze the generalization error. For simplicity, we denote the $T$ finite-sample approximation of 
\begin{eqnarray*}
L(V, \pi, \nu) &=&\EE[\phi_{V, \pi, \nu} (s, a, R, s')]:=\EE\sbr{2\nu(s, a)\rbr{R(s, a) + \gamma V(s') - V(s) - \lambda\log\pi(a|s)} - \nu^2(s, a)},
\end{eqnarray*}
as
\begin{eqnarray*}
\widehat L_T(V, \pi, \nu) &=& \frac{1}{T}\sum_{i=1}^T \phi_{V, \pi, \nu} (s, a, R, s'):= \frac{1}{T}\sum_{i=1}^T \rbr{2\nu(s_i, a_i)\rbr{R(s_i, a_i) + \gamma V(s'_i) - V(s_i)- \lambda\log\pi(a_i|s_i)} - \nu^2(s_i, a_i)},
\end{eqnarray*}
where the samples $\cbr{(s_i, a_i, s'_i, R_i)}_{i=0}^T$ are sampled \iid\, or from $\beta$-mixing stochastic process.

By definition, we have, 
\begin{eqnarray*}
&& \ell_w(\widehat V_w^*, \widehat\pi_w^*) -  \ell_w(V_w^*, \pi_w^*) \\
&=& \max_{\nu\in\Hcal_w}L_w\rbr{\widehat V_w^*, \widehat\pi_w^*, \nu} -\max_{\nu\in\Hcal_w}L_w\rbr{V_w^*, \hat\pi_w^*, \nu}\\
&=&L_w\rbr{\widehat V_w^*, \widehat\pi_w^*, \nu_w} - L_w\rbr{V_w^*, \hat\pi_w^*, \nu_w} + \underbrace{L_w\rbr{V_w^*, \hat\pi_w^*, \nu_w} - \max_{\nu\in\Hcal_w}L_w\rbr{V_w^*, \hat\pi_w^*, \nu}}_{\le 0}\\
&\le&L_w\rbr{\widehat V_w^*, \widehat\pi_w^*, \nu_w} - L_w\rbr{V_w^*, \hat\pi_w^*, \nu_w}\\
&\le&2\sup_{V,\pi,\nu\in\Fcal_w\times\Pcal_w\times\Hcal_w}\abr{\widehat L_T\rbr{V, \pi, \nu} - L_w\rbr{V, \pi, \nu}} 
\end{eqnarray*}
where $\nu_w = \max_{\nu\in\Hcal_w}L_w\rbr{\widehat V_w^*, \widehat\pi_w^*, \nu}$.

The latter can be bounded by covering number or Rademacher complexity on hypothesis space $\Fcal_w\times\Pcal_w\times\Hcal_w$ with rate $\Ocal\rbr{\sqrt\frac{\log T}{{T}}}$ with high probability if the samples are $\iid$\, or from $\beta$-mixing stochastic processes~\citep{AntSzeMun08}.

We will use a generalized version of Pollard's tail inequality to $\beta$-mixing sequences, \ie,
\begin{lemma}\label{lemma:beta_tail}[Lemma 5,~\citet{AntSzeMun08}]
Suppose that $z_1, \ldots, Z_N\in\Zcal$ is a stationary $\beta$-mixing process with mixing coefficient $\cbr{\beta_m}$ and that $\Gcal$ is a permissible class of $\Zcal\rightarrow[-C, C]$ functions, then, 
\begin{eqnarray*}
\PP\rbr{\sup_{g\in\Gcal}\abr{\frac{1}{N}\sum_{i=1}^N g(Z_i) - \EE\sbr{g(Z_1)}} > \epsilon}\le &&16\EE\sbr{\Ncal_1\rbr{\frac{\epsilon}{8}, \Gcal, \rbr{Z_i'; i\in H}}}\exp\rbr{\frac{-m_N\epsilon^2}{128C^2}} + 2m_N\beta_{k_N+1},
\end{eqnarray*}
where the ``ghost'' samples $Z_i'\in \Zcal$ and $H = \cup_{j=1}^{m_N}H_i$ which are defined as the blocks in the sampling path.
\end{lemma}
The covering number is highly related to pseudo-dimension, \ie, 
\begin{lemma}\label{lemma:cover_pseudo}[Corollary 3,~\citet{Haussler95}] For any set $\Xcal$, any points $x^{1:N}\in\Xcal^N$, any class $\Fcal$ of functions on $\Xcal$ taking values in $[0, C]$ with pseudo-dimension $D_{\Fcal}<\infty$, and any $\epsilon>0$, 
$$
\Ncal\rbr{\epsilon, \Fcal, x^{1:N}}\le e\rbr{D_\Fcal + 1}\rbr{\frac{2eC}{\epsilon}}^{D_\Fcal}
$$
\end{lemma}
Once we have the covering number of $\Phi(V, \pi, \nu)$, plug it into lemma~\ref{lemma:beta_tail}, we will achieve the statistical error,

\paragraph{Theorem~\ref{thm:statistical_error}~(Stochastic error)}
\emph{Under Assumption~\ref{asmp:sample_measure}, with at least probability $1 - \delta$,
$$
\ell_w(\widehat V_w^*, \widehat\pi_w^*) - \ell_w(V_w^*, \pi_w^*)\le 2\sqrt\frac{M\rbr{\max\rbr{M/b, 1}}^{1 / \kappa}}{C_2T},
$$
where $M = \frac{D}{2}\log t + \log\rbr{e/\delta} + \log^+\rbr{\max\rbr{C_1C_2^{D/2}, \bar\beta}}$.
}
\begin{proof}
We use lemma~\ref{lemma:beta_tail} with $\Zcal = \Scal\times\Acal\times\RR\times\Scal$ and $\Gcal = \phi_{\Fcal_w\times\Pcal_w\times\Hcal_w}$. For $\forall\Phi(V, \pi, \nu)\in\Gcal$, it is bounded by $C = \frac{2}{1-\gamma}C_R + \lambda C_\pi$. Thus, 
\begin{eqnarray}\label{eq:prob_inter}
\PP\rbr{\sup_{V,\pi,\nu\in\Fcal_w\times\Pcal_w\times\Hcal_w}\abr{\frac{1}{T}\sum_{i=1}^T\phi_{V, \pi, \nu}\rbr{(s, a, s', R)_{i}} - \EE\sbr{\phi_{V, \pi, \nu}}}\ge \epsilon/2}\\
\le 16\EE\sbr{\Ncal\rbr{\frac{\epsilon}{16}, \Gcal, (Z_i'; i\in H)}}\exp\rbr{-\frac{m_t}{2}\rbr{\frac{\epsilon^2}{16C}}^2} + 2m_T\beta_{k_T}.
\end{eqnarray}
With some calculation, the distance in $\Gcal$ can be bounded,
\begin{eqnarray*}
&&\frac{1}{T}\sum_{i\in H}\abr{\phi_{V_1, \pi_1, \nu_1}(Z'_i) -\phi_{V_2, \pi_2, \nu_2}(Z'_i) }\\
&\le& \frac{4C}{T}\sum_{i\in H}\abr{\nu_1(s_i, a_i) - \nu_2(s_i, a_i)} +  \frac{2(1 + \gamma)C}{T}\sum_{i\in H}\abr{V_1(s_i) - V_2\rbr{s_i}} \\
&& + \frac{2\lambda C}{T}\sum_{i\in H}\abr{\log\pi_1(a_i|s_i) - \log\pi_2(a_i|s_i)},
\end{eqnarray*}
which leads to 
$$
\Ncal\rbr{12C\epsilon', \Gcal, (Z_i'; i\in H)}\le \Ncal(\epsilon', \Fcal_w, (Z_i'; i\in H))\Ncal(\epsilon', \Pcal_w, (Z_i'; i\in H))\Ncal(\epsilon', \Hcal_w, (Z_i'; i\in H))
$$
with $\lambda\in (0, 2]$. To bound these factors, we apply lemma~\ref{lemma:cover_pseudo}. We denote the psuedo-dimension of $\Fcal_w$, $\Pcal_w$, and $\Hcal_w$ as $D_V$, $D_\pi$, and $D_\nu$, respectively. Thus, 
$$
\Ncal\rbr{12C\epsilon', \Gcal, (Z_i'; i\in H)}\le e^3\rbr{D_V + 1}\rbr{D_\pi + 1}\rbr{D_\nu + 1}\rbr{\frac{4eC}{\epsilon'}}^{D_V + D_\pi + D_\nu},
$$
which implies
$$
\Ncal\rbr{\frac{\epsilon}{16}, \Gcal, (Z_i'; i\in H)}\le e^3\rbr{D_V + 1}\rbr{D_\pi + 1}\rbr{D_\nu + 1}\rbr{\frac{768eC^2}{\epsilon'}}^{D_V + D_\pi + D_\nu} = C_1\rbr{\frac{1}{\epsilon}}^D,
$$
where $C_1 = e^3\rbr{D_V + 1}\rbr{D_\pi + 1}\rbr{D_\nu + 1}\rbr{768eC^2}^D$ and $D = D_V + D_\pi + D_\nu$, \ie, the ``effective'' psuedo-dimension. 

Plug this into Eq.~\eq{eq:prob_inter}, we obtain
\begin{eqnarray*}
&&\PP\rbr{\sup_{V,\pi,\nu\in\Fcal_w\times\Pcal_w\times\Hcal_w}\abr{\frac{1}{T}\sum_{i=1}^T\phi_{V, \pi, \nu}\rbr{(s, a, s', R)_{i}} - \EE\sbr{\phi_{V, \pi, \nu}}}\ge \epsilon/2} \\
&\le& C_1\rbr{\frac{1}{\epsilon}}^D\exp\rbr{-4C_2m_t\epsilon^2} + 2m_T\beta_{k_T},
\end{eqnarray*}
with $C_2 = \frac{1}{2}\rbr{\frac{1}{8C}}^2$. If $D\ge 2$, and $C_1, C_2,\bar\beta, b, \kappa>0$, for $\delta\in(0, 1]$, by setting $k_t = \lceil\rbr{C_2T\epsilon^2/b}^{\frac{1}{\kappa+1}}\rceil$ and $m_T = \frac{T}{2k_T}$, by lemma 14 in~\citet{AntSzeMun08}, we have
$$
C_1\rbr{\frac{1}{\epsilon}}^D\exp\rbr{-4C_2m_T\epsilon^2} + 2m_T\beta_{k_T}< \delta,
$$
with $\epsilon = \sqrt\frac{M\rbr{\max\rbr{M/b, 1}}^{1 / \kappa}}{C_2t}$ where $M = \frac{D}{2}\log T + \log\rbr{e/\delta} + \log^+2\rbr{\max\rbr{C_1C_2^{D/2}, \bar\beta}}$.
\end{proof}

With the statistical error bound provided in Theorem~\ref{thm:statistical_error} for solving the derived saddle point problem with arbitrary learnable nonlinear approximators using off-policy samples, we can achieve the analysis of the total error, \ie,  
\paragraph{Theorem~\ref{thm:error_decomposition}} 
\emph{Let $\hat V_w^T $ be a candidate solution output from the proposed algorithm based on off-policy samples, with at least probability $1 - \delta$, we have
\begin{eqnarray*}
\nbr{\widehat V_w^N - V^*}^2_{\mu\pi_b} &\le& \underbrace{2\rbr{6(K+C_\infty)\epsilon_{\text{app}}^{\nu} + C_\nu(1 + \gamma)\epsilon_{\text{app}}^V(\lambda) + 3C_\nu\epsilon_{\text{app}}^\pi(\lambda) } }_{\text{approximation error due to parametrization}}\\
&& + \underbrace{16\lambda^2 C^2_\pi+ \rbr{2\gamma^2 + 2}\rbr{\frac{\gamma\lambda}{1 - \gamma} \max_{\pi\in\Pcal} H(\pi)}^2 }_{\text{bias due to smoothing}}+ \underbrace{4\sqrt\frac{M\rbr{\max\rbr{M/b, 1}}^{1 / \kappa}}{C_2T}
}_{\text{statistical error}} +\underbrace{\nbr{\widehat V_w^N - \widehat V_w^*}^2_{\mu\pi_b}}_{\text{optimization error}}.
\end{eqnarray*}
where $M$ is defined as above.
}

This theorem can be proved by combining Theorem~\ref{thm:statistical_error} into Lemma~\ref{lemma:error_decomposition}.

\subsection{Convergence Analysis} \label{appsec:convergence analysis}

As we discussed in~\secref{sec:optimization error}, the \algabb~algorithm converges to a stationary point if we can achieve the optimal solution to the dual functions. However, in general, such conditions restrict the parametrization of the dual functions. In this section, we first provide the proof for~\thmref{thm:convergence_opt}. Then, we provide a variant of the \algabb~in~\algref{alg:sbeed_dual_nn}, which still achieve the asymptotic convergence with arbitrary function approximation for the dual function, including neural networks with smooth activation functions. 

\paragraph{\thmref{thm:convergence_opt}[Convergence, \citet{GhaLan13}]}
\emph{
Consider the case when Euclidean distance is used in the algorithm. Assume that the parametrized objective $\widehat\ell_{T}(V_w,\pi_w)$ is $K$-Lipschitz and variance of its stochastic gradient is bounded by $\sigma^2$. Let the algorithm run for $N$ iterations with stepsize $\zeta_k=\min\{\frac{1}{K}, \frac{D'}{\sigma\sqrt{N}}\}$ for some $D'>0$ and output $w^1,\ldots, w^N$. Setting the candidate solution to be $(\widehat V_w^N,\widehat{\pi}_w^N)$ with $w$ randomly chosen from $w^1,\ldots, w^N$ such that $P(w=w^j)=\frac{2\zeta_j-K\zeta_j^2}{\sum_{j=1}^N(2\zeta_j-K\zeta_j^2)}$, then it holds that
$\EE\sbr{\nbr{\nabla \widehat\ell_T(\widehat V_w^N,\widehat{\pi}_w^N)}^2}\leq \frac{K D^2}{N}+ (D'+\frac{D}{D'})\frac{\sigma}{\sqrt{N}}$
where $D:=\sqrt{2(\widehat\ell_T(V_w^1,\pi_w^1) -\min \widehat\ell_T(V_w,\pi_w))/K}$ represents the distance of the initial solution to the optimal solution. 
}\\
The \thmref{thm:convergence_opt} straightforwardly generalizes the convergence result in~\citet{GhaLan13} to saddle-point optimization. 
\begin{proof}
As we discussed, given the empirical off-policy samples, the proposed algorithm can be understood as solving $\min_{V_w, \pi_w} \widehat\ell_{T}(V_w,\pi_w) \defeq \widehat L_T(V_w,\pi_w;\nu^*_w)$, where $\nu_w^* = \arg\max_{\nu_w} \widehat L_T(V_w,\pi_w;\nu_w)$. 

Following the Theorem 2.1 in~\citet{GhaLan13}, as long as the gradients $\nabla_{V_w} \widehat\ell_{T}(V_w,\pi_w)$ and $\nabla_{\pi_w} \widehat\ell_{T}(V_w,\pi_w)$ are unbiased, under the provided conditions, the finite-step convergence rate can be obtained. The unbiasedness of the gradient estimator is already proved in~\thmref{thm:gradient_estimator}. 

\end{proof} 

\begin{algorithm}[t] 
\caption{{\small A variant of \algabb~learning}} \label{alg:sbeed_dual_nn}
  \begin{algorithmic}[1]
    \STATE Initialize $w=(w_V,w_\pi,w_\rho)$ and $\pi_b$ randomly, set $\epsilon$.
    \FOR{episode $i=1,\ldots, T$}
      \FOR{size $k=1,\ldots, K$}
        \STATE Add new transition $(s, a, r, s')$ into $\Dcal$ by executing behavior policy $\pi_b$.
      \ENDFOR
      \FOR{iteration $j=1, \ldots, N$}
            \STATE Sample mini-batch $\cbr{s, a, s'}^{m}\sim\Dcal$. 
            \STATE Compute the stochastic gradient w.r.t. $w_\rho$ as $G_\rho = -\frac{1}{m}\sum_{\cbr{s, a, s'}\sim\Dcal}\rbr{{\delta(s,a, s')}- \rho(s, a)}\nabla_{w_\rho}\rho(s, a)$
            \STATE Compute the stochastic gradients w.r.t. $w_V$ and $w_\pi$ as~\eqref{thm:gradient_estimator} with $w^t_\rho$, denoted as $G_V$ and $G_\pi$, respectively. 
            \STATE Decay the stepsize $\xi_j$ and $\zeta_j$.
            \STATE Update the parameters of primal function by solving the prox-mappings, \ie,
            \vspace{-2mm}
            \begin{eqnarray*}
            &&\text{update $\rho$: }\quad w^j_\rho = P_{w^{j-1}_\rho}(-\xi_j G_\rho)\\[-2mm]
            &&\text{update $V$: }\quad w^j_V = P_{w^{j-1}_V}(\zeta_j G_V)\\[-2mm]
            &&\text{update $\pi$: }\quad w^j_\pi = P_{w^{j-1}_\pi}(\zeta_j G_\pi)
            \end{eqnarray*}
            \vspace{-6mm}
        \ENDFOR
        \STATE Update behavior policy $\pi_b = \pi^N$.
    \ENDFOR
  \end{algorithmic}
\end{algorithm}

Next, we will show that in the setting that off-policy samples are given, under some mild conditions on the neural networks parametrization, the~\algref{alg:sbeed_dual_nn} will achieve a {local Nash equilibrium} of the empirical objective asymptotically, \ie, $\rbr{w^+_V, w^+_\pi, w^+_\rho}$, such that 
$$
\nabla_{w_V, w_\pi}\widehat L_\eta\rbr{w^+_V, w^+_\pi, w^+_\rho} = 0,\quad \nabla_{w_\rho}\widehat L_\eta\rbr{w^+_V, w^+_\pi, w^+_\rho} = 0.
$$ 
In fact, by applying different decay rate of the stepsizes appropriately for the primal and dual variables in the two time scales updates, the asymptotic convergence of the~\algref{alg:sbeed_dual_nn} to local Nash equilibrium can be easily obtained by applying the Theorem 1 in~\citet{heusel2017gans}, which is original provided by~\citet{borkar1997stochastic}. We omit the proof  which is not the major contribution of this paper. Please refer to~\citet{heusel2017gans,borkar1997stochastic} for further details.

\section{More Experiments}\label{appendix:more_exp}

\subsection{Experimental Details}

\paragraph{Policy and value function parametrization} The choices of the parametrization of policy are largely based on the recent paper by~\citet{RajLowTodKak17}, which shows the natural policy gradient with RBF neural network achieves the state-of-the-art performances of TRPO on MuJoCo. For the policy distribution, we parametrize it as $\pi_{\theta_\pi}(a|s) = \Ncal(\mu_{\theta_\pi}(s), \Sigma_{\theta_\pi})$, where $\mu_{\theta_\pi}(s)$ is a two-layer neural nets with the random features of RBF kernel as the hidden layer and the $\Sigma_{\theta_\pi}$ is a diagonal matrix. The RBF kernel bandwidth is chosen via median trick~\citep{DaiXieHe14,RajLowTodKak17}. Same as~\citet{RajLowTodKak17}, we use $100$ hidden nodes in InvertedDoublePendulum, Swimmer, Hopper, and use $500$ hidden nodes in HalfCheetah. This parametrization was used in all on-policy and off-policy algorithms for their policy functions. We adapted the linear parametrization for control variable in TRPO and Dual-AC following~\citet{Dai18boosting}. In DDPG and our algorithm~\algabb, we need the parametrization for $V$ and $\rho$ (or $Q$) as fully connected neural networks with two tanh hidden layers with 64 units each. 

In the implementation of~\algabb, we use the Euclidean distance for $w_V$ and the $KL$-divergence for $w_\pi$ in the experiments. We emphasize that other Bregman divergences are also applicable.

\paragraph{Training hyperparameters} For all algorithms, we set $\gamma = 0.995$. All $V$ and $\rho$ (or $Q$) functions of \algabb~and DDPG were optimized with ADAM. The learning rates were chosen with a grid search over $\{0.1, 0.01, 0.001, 0.001\}$. For the \algabb, a stepsize of $0.005$ was used. For DDPG, an ADAM optimizer was also used to optimize the policy function. The learning rate is set to be $1e-4$ was used. For \algabb, $\eta$ was set from a grid search of $\{0.004, 0,01, 0.04, 0.1, 0.04\}$ and $\lambda$ from a grid search in $\cbr{0.001, 0.01, 0.1}$. The number of the rollout steps, $k$ was chosen by grid search from $\cbr{1, 10, 20, 100}$. For off-policy \algabb, a training frequency was chosen from $\{1, 2, 3\}\times 10^3$ steps. A batch size was tuned from $\{10000, 20000, 40000\}$. DDPG updated it's values every iteration and trained with a batch size tuned from $(32, 64, 128)$. For DDPG, $\tau$ was set to $1e-3$, reward scaling was set to $1$, and the O-U noise $\sigma$ was set to $0.3$.

\subsection{On-policy Comparison in Continuous Control Tasks}

\begin{figure*}[!t]
\centering
  \begin{tabular}{cc}
    \includegraphics[width=0.3\textwidth,  trim={0.5cm 0.6cm 0.6cm 1.1cm},clip]{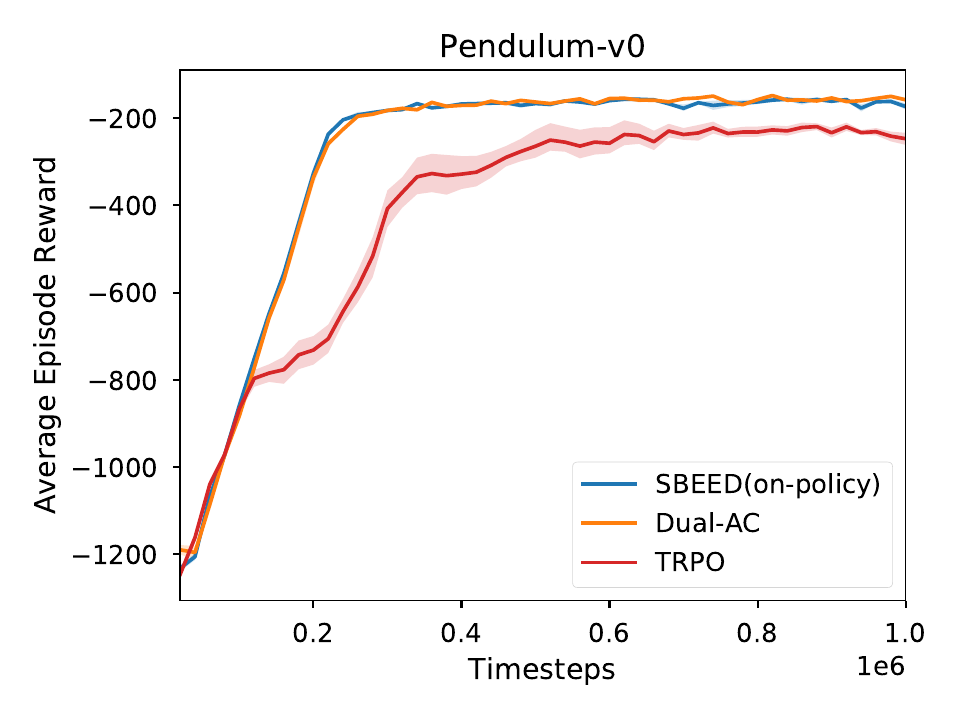}&
    \includegraphics[width=0.3\textwidth,  trim={0.5cm 0.6cm 0.6cm 1.1cm},clip]{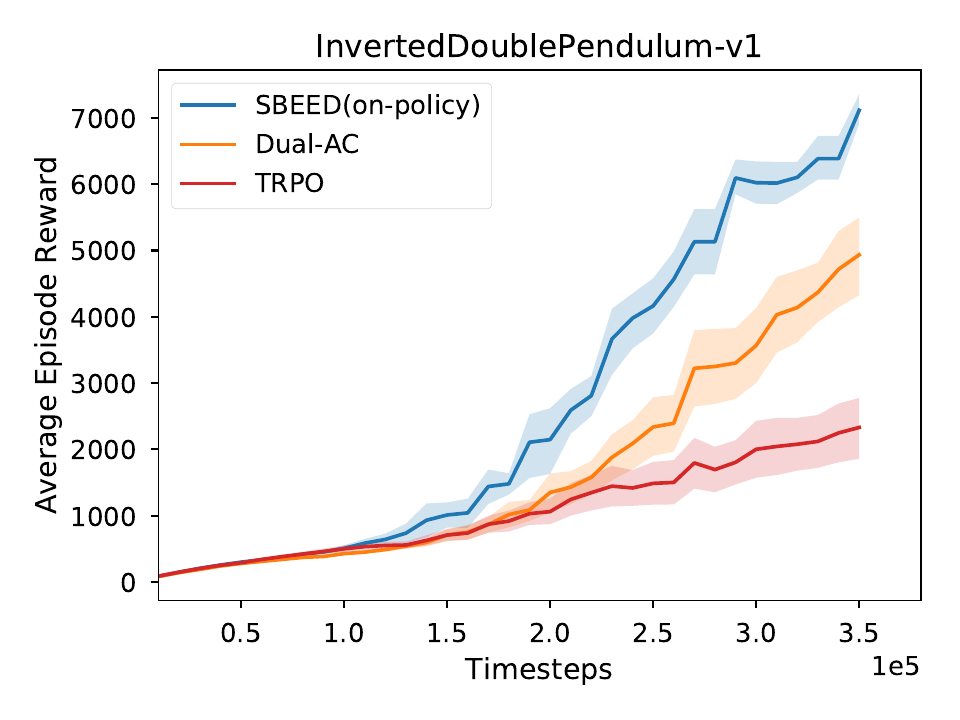}\\
    (a) Pendulum  &(b) InvertedDoublePendulum\\
    \vspace{2mm}
  \end{tabular}
  \begin{tabular}{ccc}
    \includegraphics[width=0.3\textwidth,  trim={0.5cm 0.6cm 0.6cm 1.1cm},clip]{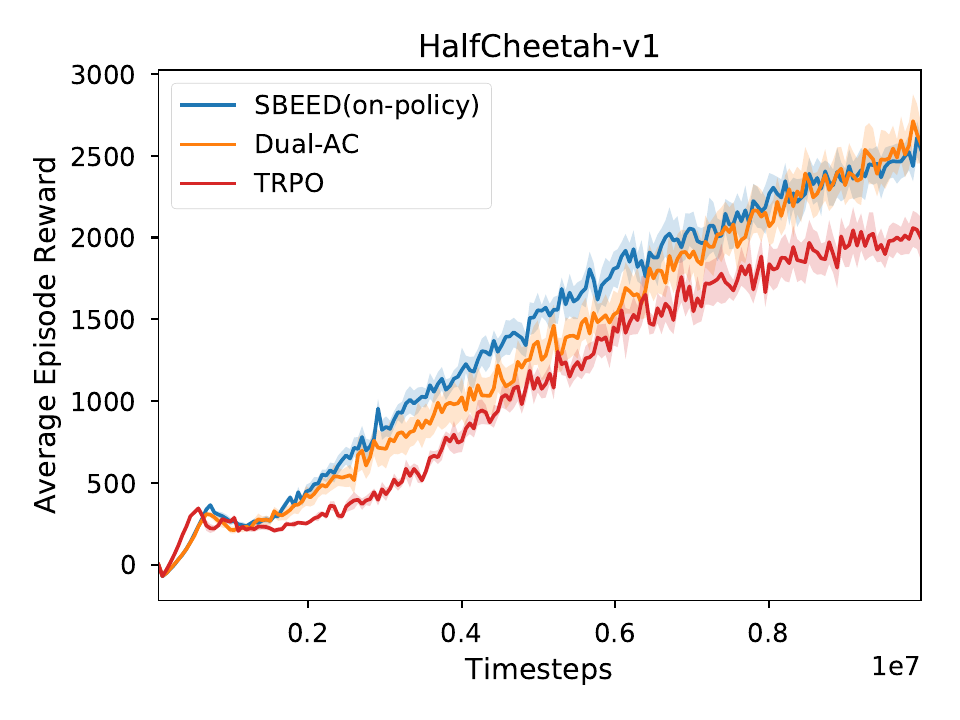}&
    \includegraphics[width=0.3\textwidth,  trim={0.5cm 0.6cm 0.6cm 1.1cm},clip]{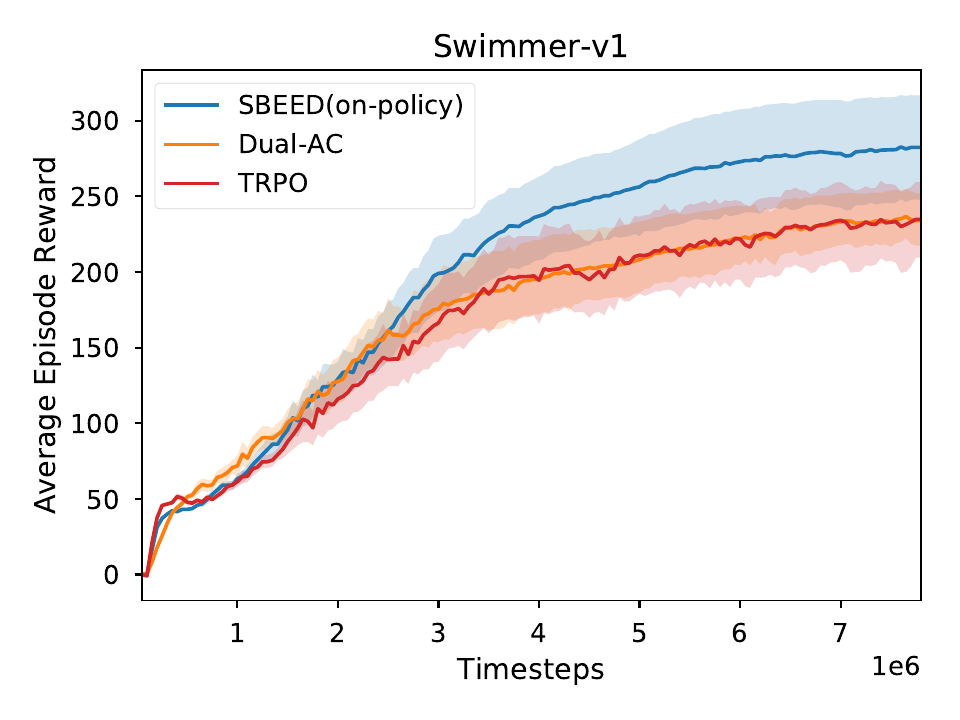}&
    \includegraphics[width=0.3\textwidth,  trim={0.5cm 0.6cm 0.6cm 1.1cm},clip]{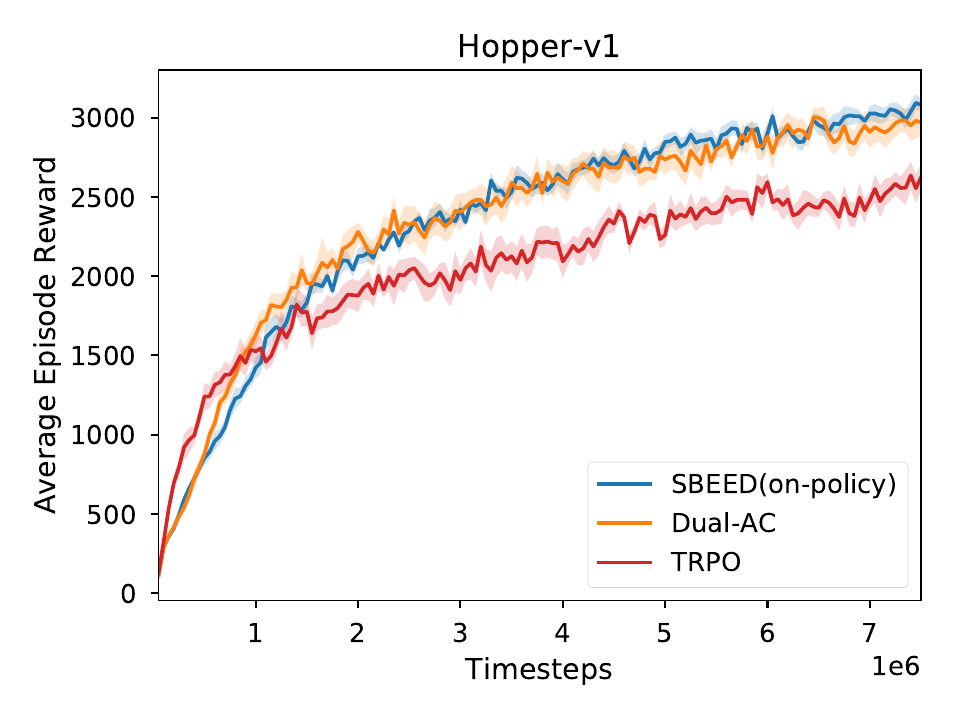}\\ 
    (c) HalfCheetah & (d) Swimmer & (e) Hopper \\
  \end{tabular}
  \caption{The results of \algabb~against TRPO and Dual-AC in on-policy setting. Each plot shows average reward during training across $5$ random runs, with $50\%$ confidence interval. The x-axis is the number of training iterations. \algabb~achieves better or comparable performance than TRPO and Dual-AC on all tasks. 
  }
  \label{fig:onpolicy_comparison}
\end{figure*}

We compared the~\algabb~to TRPO and Dual-AC in on-policy setting. We followed the same experimental set up as it is in off-policy setting. We ran the algorithm with $5$ random seeds and reported the average rewards with 50\% confidence intervals. The empirical comparison results are illustrated in Figure~\ref{fig:onpolicy_comparison}. We can see that in all these tasks, the proposed \algabb~achieves significantly better performance than the other algorithms. This can be thought as another ablation study that we switch off the ``off-policy'' in our algorithm. The empirical results demonstrate that the proposed algorithm is more flexible to way of the data sampled. 

We set the step size to be $0.01$ and the batch size to be $52$ trajectories in each iteration in all algorithms in the on-policy setting. For TRPO, the CG damping parameter is set to be $10^{-4}$.

\end{appendix}

\end{document}